\documentclass[review]{elsarticle}

\usepackage{hyperref}
\usepackage{lineno}
\modulolinenumbers[5]

\journal{Journal of \LaTeX\ Templates}

\usepackage{colortbl}

\usepackage{bm}
\usepackage{todonotes}
\usepackage{amssymb}
\usepackage{amsmath}
\usepackage{amsthm}

\usepackage{subcaption}
\usepackage{hyperref}
\usepackage{textcomp}
\usepackage{siunitx}
\usepackage{rotating}
\usepackage{multirow}

\DeclareMathOperator*{\argminA}{\arg\min}

\theoremstyle{definition}
\newtheorem{definition}{Definition}[section]

\definecolor{mygray}{gray}{0.9}

\usepackage{tikz}

\usepackage{multirow}
\usepackage{pifont}

\usepackage{amssymb}
\usepackage{xcolor}
\newcommand{\redcheck}{{\color{red}\ding{51}}}
\newcommand{\greencheck}{{\color{green}\ding{51}}}
\newcommand{\blackcheck}{{\bf \color{black}\ding{51}}}

\newcommand{\redxmark}{{\color{red}\ding{55}}}%
\newcommand{\greenxmark}{{\color{green}\ding{55}}}%

\newtheorem{theo}{Proposition}[section]
\newtheorem*{propApp}{Proposition}

\newtheorem{hyp}{Hypothesis}

\begin{document}

\begin{frontmatter}

\title{A New Similarity Space Tailored for Supervised Deep Metric Learning}

\author[ufmg]{Pedro H. Barros\corref{cor1}, \corref{cor3}}
\ead{pedro.barros@dcc.ufmg.br}

\author[ufal]{Fabiane Queiroz\corref{cor2}}\ead{fabiane.queiroz@laccan.ufal.br}

\author[ufmg]{Flavio Figueredo\corref{cor1}}\ead{flaviovdf@dcc.ufmg.br}
\author[ufmg]{Jefersson A.\ dos Santos\corref{cor2}}\ead{jefersson@dcc.ufmg.br}
\author[ufmg]{Heitor S.\ Ramos\corref{cor2}}\ead{ramosh@dcc.ufmg.br}

\address[ufmg]{Departamento de Ci\^{e}ncia da Computa\c{c}\~{a}o, Universidade Federal de Minas Gerais (UFMG)\\Belo Horizonte, MG, Brazil - CEP 31270-901}

\address[ufal]{Centro de Ci\^encias Agr\'arias,, Universidade Federal de Alagoas (UFAL)\\ Macei\'o, AL, Brazil - CEP 57072-900 }

\cortext[cor3]{Corresponding author}


\begin{abstract}
We propose a novel deep metric learning method. Differently from many works on this area, we defined a novel latent space obtained through an autoencoder. The new space, namely S-space, is divided into different regions that describe the positions where pairs of objects are similar/dissimilar. We locate makers to identify these regions. We estimate the similarities between objects through a kernel-based t-student distribution to measure the markers' distance and the new data representation. In our approach, we simultaneously estimate the markers' position in the S-space and represent the objects in the same space.
Moreover, we propose a new regularization function to avoid similar markers to collapse altogether. We present evidences that our proposal can represent complex spaces, for instance, when groups of similar objects are located in disjoint regions. We compare our proposal to 9 different distance metric learning approaches (four of them are based on deep-learning) on 28 real-world heterogeneous datasets. According to the four quantitative metrics used, our method overcomes all the nine strategies from the literature.
\end{abstract}


\end{frontmatter}


\section{Introduction}

A distance metric is a function that provides a way to measure how far apart two elements of a set are from each other. 
Among various works involving machine learning applications, the most commonly used metric is the Euclidean distance~\cite{de2000mahalanobis}. Methods that use Euclidean distance usually consider that all variables' covariance is zero, i.e., there is no correlation among them, but this assumption is hardly found in the real world~\cite{XIANG20083600}. Euclidean distance and cosine similarity are popular for many applications. For instance, the cosine similarity is vastly used for text mining~\cite{6420834}. Even showing its effectiveness in several applications, the cosine similarity assumes equal weight for every dimension, limiting its application~\cite{6420834}.


Euclidean and cosine distance are known as data-independent techniques, once  they are defined without any prior knowledge about the data. Learning distances, a.k.a Metric Learning (MeL), from data is a common attempt to improve machine learning approaches~\cite{Weinberger:2009:DML:1577069.1577078,Liu:2018:ESR:3219819.3220031, NIPS2018_7377, Huai:2018:MLP:3219819.3219976, Inaba:2019:FEB:3292500.3330975, NIPS2004_2566, WANG2019202}. In modern machine learning research, MeL is a fundamental technique for several different applications such as sorting~\cite{mcfee2010metric}, classification (e.g., k-nearest neighbors), clustering~\cite{5611527}, and ranking~\cite{pmlr-v80-vogel18a}.

MeL aims to estimate distance function parameters based on a given training set. A common approach is to frame MeL as a convex optimization problem~\cite{NIPS2002_2164}. Thus, a distance $d$ can be defined as $ d_{\bm M} (x, y) = \sqrt {(x - y) ^ T \bm M (x - y)}$, in which $\bm M$ is a positive semi-definite matrix. In the case $\bm M$ is the covariance matrix, we have the \textit{Mahalanobis} distance~\cite{de2000mahalanobis}. Classic methods proposed for metric learning use $ d_{\bm M}$ to search for the best linear space that captures the semantics of the data (e.g., in a classification setting, we search for $\bm M$  that minimizes the miss-classification loss). However, the linear transformation has some limitations, as it cannot model high-order correlations between the original data dimensions~\cite{CAO2019217}.


 
Using MeL, we can define metrics that consider the covariance of attributes. Additionally, MeL approaches do not necessarily assume linear relationships, although classical MeL techniques like the {\it Mahalanobis} distance~\cite{de2000mahalanobis} assumes a linear space. Moreover, MeL does not assume equal weights for every attribute~\cite{6420834}. The assumption that MeL can be treated as a convex optimization problem can also be relaxed using the appropriate model.

To tackle the issues mentioned above, deep learning techniques are currently being used for MeL~\cite{Hou_2019_CVPR, Zheng_2019_CVPR,Niethammer_2019_CVPR,NIPS2019_8339,Shen_2020_WACV,Paixao_2020_CVPR}. Since these proposals seek to learn a non-linear feature representation, they usually overperform standard techniques found in the literature. Neural Networks (NNs) are natural candidates and are typically used to learn similarity metric ~\cite{drlim, 1467314}.

\begin{figure*}
\centering
\begin{subfigure}{.32\textwidth}
  \centering
  \includegraphics[width=1.0\linewidth]{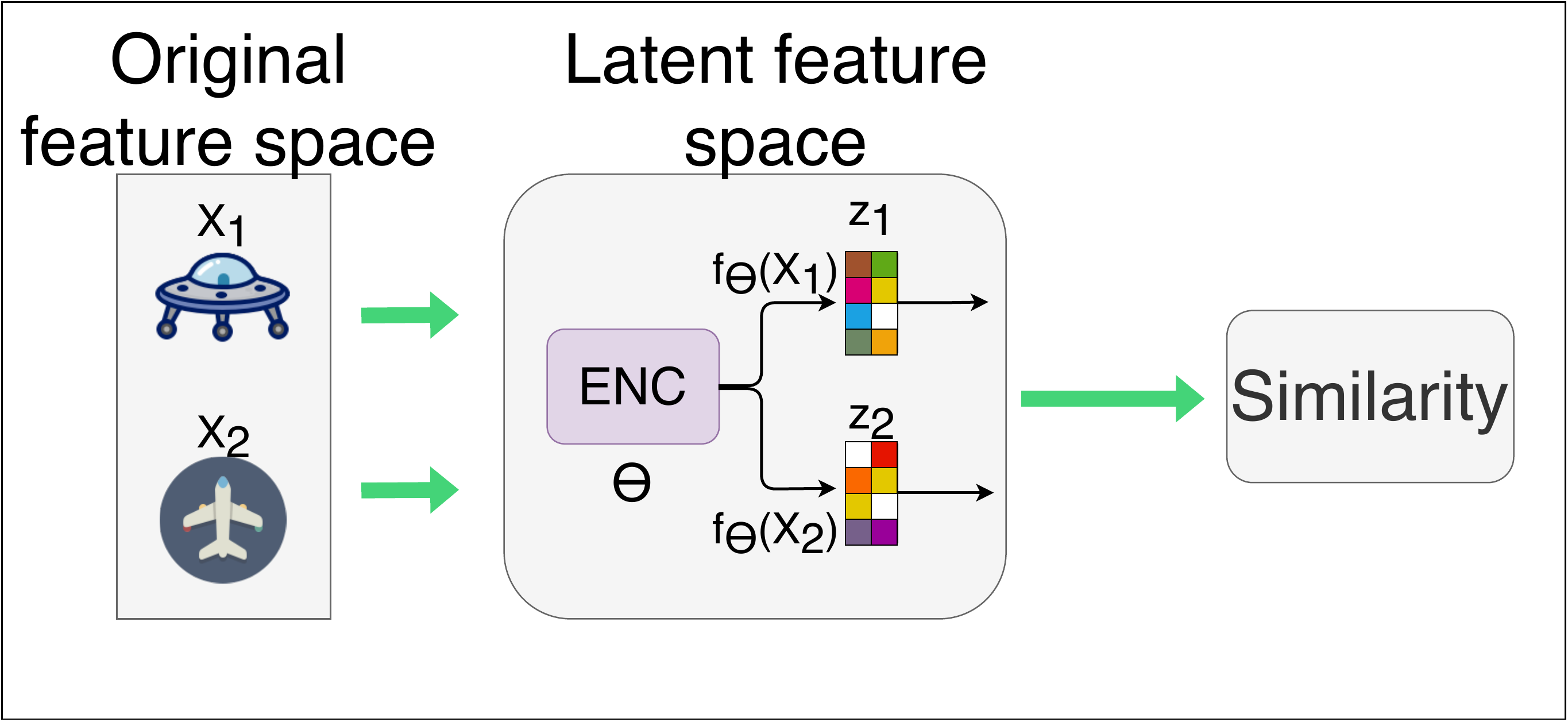}
  \caption{Canonical scheme of DMeL.}
  \label{fig:siamese_canonical}
\end{subfigure}
\begin{subfigure}{.461\textwidth}
  \centering
  \includegraphics[width=1.0\linewidth]{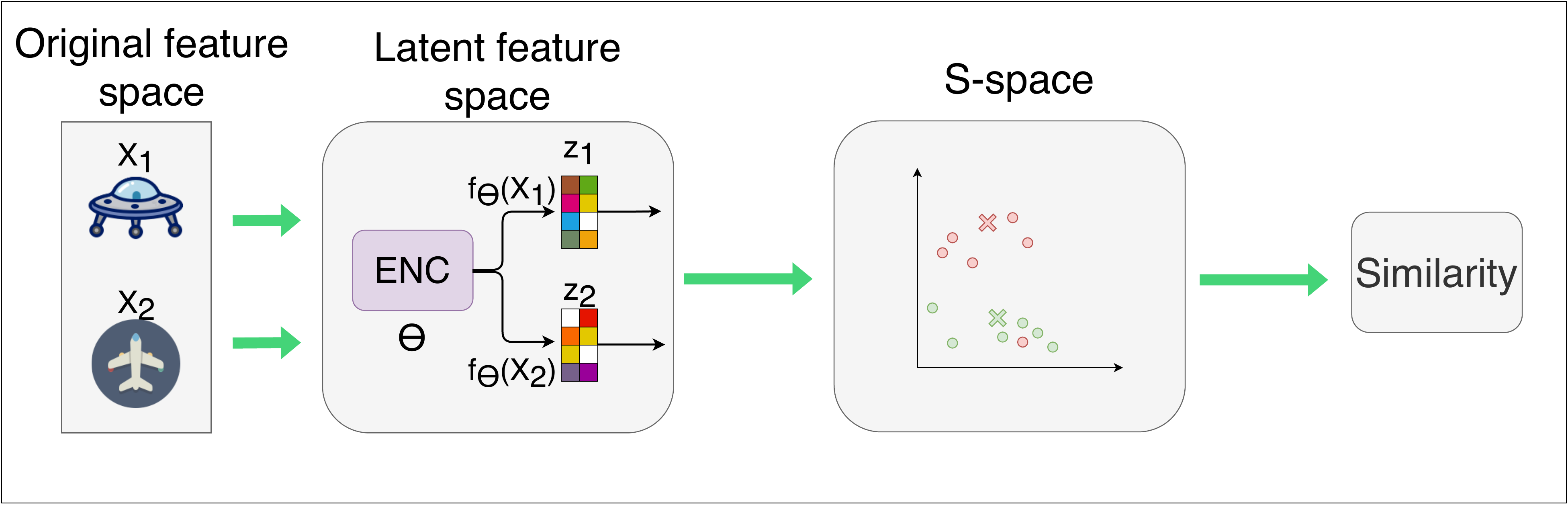}
  \caption{Our scheme of DMeL with S-space.}
  \label{fig:siamese_our}
\end{subfigure}

\caption{Comparison between the canonical model of DMeL and the model used in this work.}
\label{fig:siamese}
\end{figure*}

The representation of compressed data found by a Neural Network (NN) is commonly named as latent feature space and the data in this space as latent data, as we can see in Figure~\ref{fig:siamese_canonical}. 
Our work hypothesizes that the latent feature space captured by NNs can be improved with an auxiliary space. For instance, common NNs-based Deep Metric Learning (DMeL) approaches extract a latent space that encodes similar and dissimilar {\it points}, but not the separability between them. However, this single representation is limited, as it does not capture {\it pairwise information}. 

Unlike the literature, our approach employs NNs, fed by labeled original pairwise data, to find a {\it latent pairwise space with markers}. This approach is shown in Figure~\ref{fig:siamese_our} as we now detail. In our method, data comes in pairs of vectors $(\mathbf{x}_i, \mathbf{x}_j)$ which are deemed as similar ($y_{ij}=+$) or dissimilar ($y_{ij}=-$).
The first part of our architecture is an autoencoder. After encoding the pair of input objects, our major novelty is on converting {\it data pairs} to a new Similarity-space (called S-space). A data point $\mathbf{x}_i$ is mapped into $\mathbf{z}_i = f_\Theta(\mathbf{x}_i)$ in the latent space, where $\Theta$ are model parameters. The S-Space, for a pair of points $i$ and $j$ is composed of two novel ideas. Firstly, we represent points as a similarity vector between pairs, i.e., $\mathbf{s}_{ij} = |\mathbf{z}_i - \mathbf{z}_j|$. Secondly, {\it and more importantly}, we define markers that act as reference points to similar (${\bm \mu}^+_p \in \mathcal{M}^+$) and dissimilar (${\bm \mu}^-_n \in \mathcal{M}^-$) regions. Markers' position are learned in the optimization process.

Our loss function is comprised of three parts. Firstly, an autoencoder loss function takes care of data encoding and decoding. The second loss function captures the sum of distances between similarity vectors ($\mathbf{s}_{ij}$) and markers (${\bm \mu}_m \in \mathcal{M}^+ \cup \mathcal{M}^-
)$, in this work, we used a T-student kernel to estimate this distance and we apply
 a cross-entropy loss function between the input labels and the model output.
The last part of our loss function is called a {\it repulsive regularizer}. It is inversely proportional to the distance of the markers of the same class. 
This loss function ensures that markers are different (the loss increases as markers become similar), ensuring some diversity level on the marker set. It attempts that markers capture complex similarity regions such as disjoint similarity/dissimilarity regions.



We named our approach as \textit{Supervised Distance Metric learning Encoder with Similarity Space} (SMELL). Our method is herein described as supervised learning, but it can be appropriately extended to unsupervised and semi-supervised learning. Through a wide range of experiments on 28 datasets, we show that SMELL provides gains over the state-of-the-art in all of them. To explain its accuracy, we show evidence supporting the following two hypotheses.



\begin{hyp}(\textbf{H\ref{hpy:disjoint_regions}})
\label{hpy:disjoint_regions}
Using SMELL, the markers group data points considered similar (in our context, which have the same labels) and dissimilar (different labels) into disjoint regions in S-space. 
\end{hyp}

\begin{hyp}(\textbf{H\ref{hyp:separability}})
\label{hyp:separability}
SMELL increases the input pairs' separability in the latent feature space for different types of pairs (similar/dissimilar).
\end{hyp}
%


Overall, the main contributions of our work are:

\begin{enumerate}[(i)]
    \item a new data representation space called \textit{Similarity space} (S-space) that separates regions where similar/dissimilar objects lie together and help the convergence of the model. We also investigate interpretability and data visualization in this space. S-space can capture complex regions that can model similar points in disjoint regions;
    \item a new distance metric learning method that simultaneously learns a latent representation of the data and the markers' position in the S-space;
    \item we found evidence that the number of markers is a virtual hyperparameter of the model and does not need to be tuned.
    \item a new regularization function to avoid model overfitting called \textit{repulsive regularizer}.
\end{enumerate}

This paper is organized as follows: 
Section~\ref{sec:related} presents the related works to distance metric learning;
Section~\ref{sec:notations} describes some notations and a background review for the good understanding our proposal;  Section~\ref{sec:our} describes our proposal; 
Section~\ref{sec:experiments} describes the experimental setup used to analyze the data; 
Section~\ref{sec:discussion} presents the main results and discussions and Section~\ref{sec:conclusion} concludes.

\section{Related work}
\label{sec:related}


In the distance metric learning task, prior research usually assumes that the datasets are represented by an incomplete set of features (i.e., we can never collect all the features of an object). This subset of features may not thoroughly inform the semantics of the data space. Thus, the objective is to learn a similarity matrix that encodes how these features should be combined to compute distances best.

One of the first successful cases to solve this problem was learning the linear matrix (Mahalanobis) metric to find a new representation in the feature space~\cite{Xing:2002:DML:2968618.2968683, Globerson:2005:MLC:2976248.2976305, KAN2020107086}. This paradigm requires the decomposition of eigenvalues, an operation that is cubic in the dataset dimensionality (i.e., number of features). This issue severely impacts the training time. Also, approaches like this one are limited to similarity matrices, which encode linear combinations of features. 

Other approaches proposed techniques based on Information Theory to tackle the distance metric learning problem~\cite{788121,NGUYEN2017215,Davis:2007:IML:1273496.1273523}. 
These works start from a reference distribution to train distance functions based on divergences (e.g., Kullback-Leibler or Jeffrey) to obtain reference probability distributions of the data. Through this reference distribution, the authors estimate the similarity. 
These methods usually suffer from convergence issues~\cite{Davis:2007:IML:1273496.1273523} when optimizing.


In kernel-based methods for distance metric learning, the input data is usually transformed into a higher dimensional space. The algorithm learns object similarities using the new space obtained from the kernel function~\cite{788121, NIPS2006_3088,7502076, 8617698}. These methods also suffer from a cubic computation cost (on the number of features) or suffer from convergence issues, limiting its applicability due to training time.


In the context of Deep Learning methods, a typical family of Deep Neural Network models that learns distance metrics is the Siamese Neural Networks (SNNs). One of the first works using this approach can be seen in~\cite{bromley1994signature}, where the authors propose a model composed of two neural networks that share their weights among themselves. This architecture was initially proposed for the signature verification problem. 

Neural Networks (NNs) seek to find nonlinear similarities between comparable data examples by extracting a feature vector representing the difference between the data examples. There are several works in the context of NNs developed for different applications~\cite{chopra2005learning, koch2015siamese, Shen:2017:DSN:3123266.3123452,geosciencepaper, ijcai2019-745, WU2020}. They are easily scalable (do not suffer from the cubic cost as before), as they do not explore eigenvalues decomposition. NNs are typically optimized with functions that consider pairs of inputs, called pairwise loss function, and these proposals tend to find a new representation of the data. Therefore, a similarity function is defined in this new representation (for example, Euclidean). 

More recent works present deep metric learning with contrastive loss~\cite{drlim,1467314} and triplet loss~\cite{Schroff_2015_CVPR}. Even showing promising results, these proposals present some issues, such as slow convergence and poor local optima, optimizing the model a challenging task. Contrastive embedding is highly dependent on the quality of the representation of the training data. The training set must contain real-valued precision for pairwise samples. This consideration is typically difficult to satisfy, which is usually not available in practice~\cite{Wang_2017_ICCV}. For the triplet model, the loss function defines an inequality between positive and negative examples for a given anchor example. These methods suffer from what is called the hard negative problem~\cite{wu2017sampling, NGUYEN2020209}. Here, some specific negative examples deteriorate the quality of the model, making the training unstable~\cite{Cui_2016_CVPR}. Hard negative data mining is a proposal to work around this problem. However, the computational cost of searching for these examples becomes high. In addition, it is unclear what defines ``good'' hard triplets~\cite{10.1007/978-3-319-46448-0_44}.

Recent work, including N-pair loss~\cite{Nparloss}, Lifted Structure~\cite{Song_2016_CVPR}, and the Multi-Similarity Loss~\cite{Wang_2019_CVPR} propose strategies to capture relationships within a mini-batch selection. Typically, these strategies consider a weight function that associates the pairs of elements in the loss calculation.  Nevertheless, these works are based on distance measurements between pairs of similar and dissimilar objects in the space found by the neural network.


The methods mentioned in this section indicate the feasibility of learning a similarity function from the input data. Some of these methods inspire the present work~\cite{bromley1994signature,maaten2008visualizing}; for instance, we use a Neural Networks to extract the data representation and the t-student kernel distribution to create a similarity metric. 

However, we devised a novel deep metric learning method differently from the literature using a new representation space (S-space) obtained through autoencoders. As defined herein, the S-space helps the convergence of the proposal and, thanks to the possibility of having multiple markers to represent similar objects, it models even complex spaces such as noncontinuous spaces where similar objects lie in disjoints regions. Therefore, we propose a new similarity space that helps the learning of autoencoders. Unlike pairwise loss, our proposal does not require any specific sample selection strategy.

\section{Background and notation}
\label{sec:notations}
 
In SMELL, we map pairwise input data into a latent space and a Similarity space. In this section, we provide some technical background about data representation with autoencoders and a mathematical notation essential to the proposed method understanding.


Throughout the paper, we apply the following notation. We denote vectors by boldface lowercase letters, such as $\bm{x}$, $\bm{z}$ and $\bm \mu$; all scalars by lowercase letters, such as $m$ and $n$; sets of parameters by greek uppercase letters, such as $ \Theta$ and $\Sigma$; and sets by calligraphic uppercase letters, such as $\mathcal{X}$ and  $\mathcal{Z}$. The zero-mean normal distribution will be denoted by  $\mathcal N(\mu=0, \sigma)$. Table ~\ref{tab:resumo_var} summarizes this notation. 

 \begin{table}
\centering
\small
\begin{tabular}{cc}
\hline
Notation & Description \\ \hline
$\mathcal{X}$       &   input data  examples set \\ 
$\bm x_i$       &   m-dimensional single element in $\mathcal{X}$           \\
$\mathcal{Y}$       &    Label set for set $\mathcal{X}$ \\
$y_i$       &    single element in $\mathcal{Y}$         \\
$\mathcal{Z}$       &  Latent Feature Space from $\mathcal{X}$\\
$\bm z_i$       &   n-dimensional single element in $\mathcal{Z}$      \\
$f_\Theta$       &    Encoder function        \\
$\Theta$       &    set of weights for encoder         \\
$f_{\Theta'}$       &    Decoder function          \\
$\Theta'$       &    set of weights for decoder        \\
$l$       &     label function for a element in set $\mathcal{X}$         \\
$l'$       &    label function for a element in set $\mathcal{Z}$     \\ 
$\mathcal{S}$       &    The \textit{similarity space} from $\mathcal{Z}$       \\
$\bm s_{ij}$       &    n-dimensional single element in $\mathcal{S}$  \\
$\mathcal{M}$       &    The \textit{markers set},  subset of $S$       \\
$\mu_i$      &     n-dimensional single element in $\mathcal{M}$     \\
$f^S$       &    Function that maps a pair in $\mathcal{X}$ to an element in $\mathcal{S}$       \\
$\psi$       &    The similarity function       \\
$\Sigma$       &    Set of parameters of $\psi$ ($\Theta$, $\Theta'$ and $\mathcal{M}$)          \\

\hline
\end{tabular}
\caption{Notation used in this article.}
\label{tab:resumo_var}
\end{table}

 


Let the set $\mathcal{X} = {\{\bm{x}_i\}}_{i = 1}^v$, with $\bm{x}_i \in \mathbb{R}^m$, be $v$ data examples defined in an $m$-dimensional feature space. For each $\bm{x}_i \in \mathcal{X}$ there is an associated label $y_i \in \mathcal{Y} = {\{y_i\}}_{i = 1}^v$, where $y_i \in \{1,...,b\}$. In this way, the pair $(\bm x_i,y_i)$ indicates which of $b$ classes a input $\bm x_i$ belongs to. In a supervised Machine Learning classification problem, we seek to find a function $l:\mathcal{X} \rightarrow \mathcal{Y}$ that maps an unlabeled example $\bm x_i$ into their respective label $y_i$. 

To develop the proposed work, we introduce here some important definitions:

\theoremstyle{definition}
\begin{definition}{(\textit{The latent feature space})}
\label{def:latentspace} 
Consider the set $ \mathcal {X} $ as the original feature space and the representation function $ f_\Theta:\mathcal{X} \xrightarrow{} \mathcal{Z}$, in which $f_\Theta(\bm x_i) = \bm z_i \Longrightarrow l(\bm x_i) = l'(\bm z_i) = y_i$ and the function $l': \mathcal{Z} \xrightarrow{} \mathcal{Y}$, which maps the latent data into their respective labels. We can defined the representation space  $\mathcal{Z}$ called \textit{latent feature space} from $\mathcal{X}$ as
$\mathcal{Z} = \{\bm z_i\}_{i=1}^v, \text{ with } \bm z_i \in \mathbb{R}^n.$
\end{definition}





An autoencoder is a Neural Network trained to attempt to copy a data input to its output. It can be seen as consisting of two parts: an encoder and a decoder that produces an input-based reconstruction~\cite{Goodfellow-et-al-2016}.  
%
An encoder is a representation learning algorithm that seeks to find a representation function $f_ \Theta:\mathcal{X} \xrightarrow{} \mathcal{Z}$ for a set of weights $ \Theta$
that maps the set~$\mathcal{X}$ to the \textit{latent feature space}~$\mathcal{Z}$.

Similarly, the decoder function can be defined as the inverse encoder function $ f^{-1}_{ \Theta'}: \mathcal{Z} \xrightarrow{} \mathcal{X} $ where $ \Theta' $ is a set of weights for the decoder. Autoencoders are trained to minimize reconstruction errors (typically, Mean Squared Errors - MSE), and its training is performed through \textit{Backpropagation} of the error, just like a regular Feedforward Neural Network~\cite{hecht1992theory}.  




A neural network model~\cite{bromley1994signature, wang2019multi} receives a pair of input examples $(\bm x_i, \bm x_j)\in \mathcal{X} \times \mathcal{X}$ and transform each of them to a latent data $(\bm z_i, \bm z_j) \in \mathcal{Z} \times \mathcal{Z}$ through the encoder $f_\Theta$. 
%
%
%

%

In the context of supervised learning, for a data pairwise $(\bm x_i, \bm x_j) \in \mathcal{X} \times \mathcal{X}$,  we say they are similar iff $l(\bm x_i) = l(\bm x_j)$. Analogously, they are dissimilar iff $l(\bm x_i) \neq l (\bm x_j)$.
%
%

\theoremstyle{definition}


\begin{definition}{(\textit{The similarity space})}
\label{def:sspace} The representation space called \textit{Similarity space} (or \textit{S-space}) is a space built from the set $\mathcal{X} \times \mathcal{X}$. So, be the function $f^{S}:{\mathcal{X} \times \mathcal{X}} \rightarrow \mathcal{S}$, the \textit{similarity space} is defined as $\mathcal{S} = \{\bm s_{ij}\} \text{, with } \bm s_{ij} \in S \subset \mathbb{R}^n,$
where if $ l (\bm x_i) = l (\bm x_j) $, then $ \bm s_{ij} $ represents the similarity vector and if $ l (\bm x_i) \neq l(\bm x_j) $, then $ \bm s_{ij} $ represents the dissimilarity vector.
\end{definition}

In this paper, we define the map function $f^{S}: {\mathcal{X} \times \mathcal{X}} \rightarrow \mathcal{S}$ for a pairwise $(\bm x_i,\bm x_j)$ by the following element-wise absolute value operation:

\begin{equation}
\begin{split}
   \bm s_{ij} & = f^{S}(\bm x_i, \bm x_j)\\
              & =  |f_{ \Theta}(\bm x_i) - f_{\Theta}(\bm x_j)|\\ 
              & = |\bm z_i - \bm  z_j| \\
              & = (|z^1_i-z^1_j|, | z^2_i-z^2_j|, ..., | z^n_i- z^n_j|)
\end{split}
\label{eq:elementwise}
\end{equation}

\noindent it is worth noting that since $\bm s_{ij}$  is obtained by an element-wise process, it has the same dimension as $\bm z_i $ and $\bm z_j$, where $z^n_i$ is the n-th feature of the i-th data example in a latent space representation $\mathcal{Z}$ (see Definition~\ref{def:latentspace}). 

\theoremstyle{definition}
\begin{definition}{(\textit{The Markers set})}
\label{def:markersset} In \textit{S-space}, we defined the markers set  $\mathcal{M}  \subset \mathbb{R}^n$ (same space then $\mathcal{S}$) to improve similarity calculations. We define the set  $ \mathcal{M}^+ $ representing the set of markers responsible for quantifying the similarity between the input pairs. Likewise, markers in set $ \mathcal{M}^- $ quantify the dissimilarity.
The Markers set is defined as

\begin{equation}
    \mathcal{M}  = \mathcal{M}^+ \cup  \mathcal{M}^- =  \{\bm \mu_i^+\}_{i=1}^{k} \cup \{\bm \mu_j^-\}_{j=k+1}^{w}.
    \label{eq:markers}
\end{equation}

\end{definition}

Therefore, in this work, we seek to calculate the similarity function \linebreak $\psi_ \Sigma:\mathcal{X}\times\mathcal{X}\xrightarrow{}[0, 1]$.
The parameters of $\psi$ are defined by the set $\Sigma = \{\Theta, \Theta', \mathcal{M}\}$, respectively the weights of encoder, decoder and the Markers set in S-space. SMELL relies in simultaneously learning all elements of $\Sigma$. More details about the proposed method are described in Section \ref{sec:our}.

\section{Supervised Distance Metric learning Encoder with Similarity Space (SMELL)}
\label{sec:our}

Our proposal, namely SMELL, simultaneously optimizes a latent data representation (using a DMeL model) and a similarity function that indicates the similarity of two objects in the learned data S-space. This kind of technique can be useful for a wide variety of applications, such as to feed a predictor (e.g., a classifier) with a new metric learned from the data. This section details our proposal. Figure \ref{fig:sche_our} shows a simple schematic for our proposal.

\begin{figure}
\centering
\includegraphics[width=0.95\linewidth]{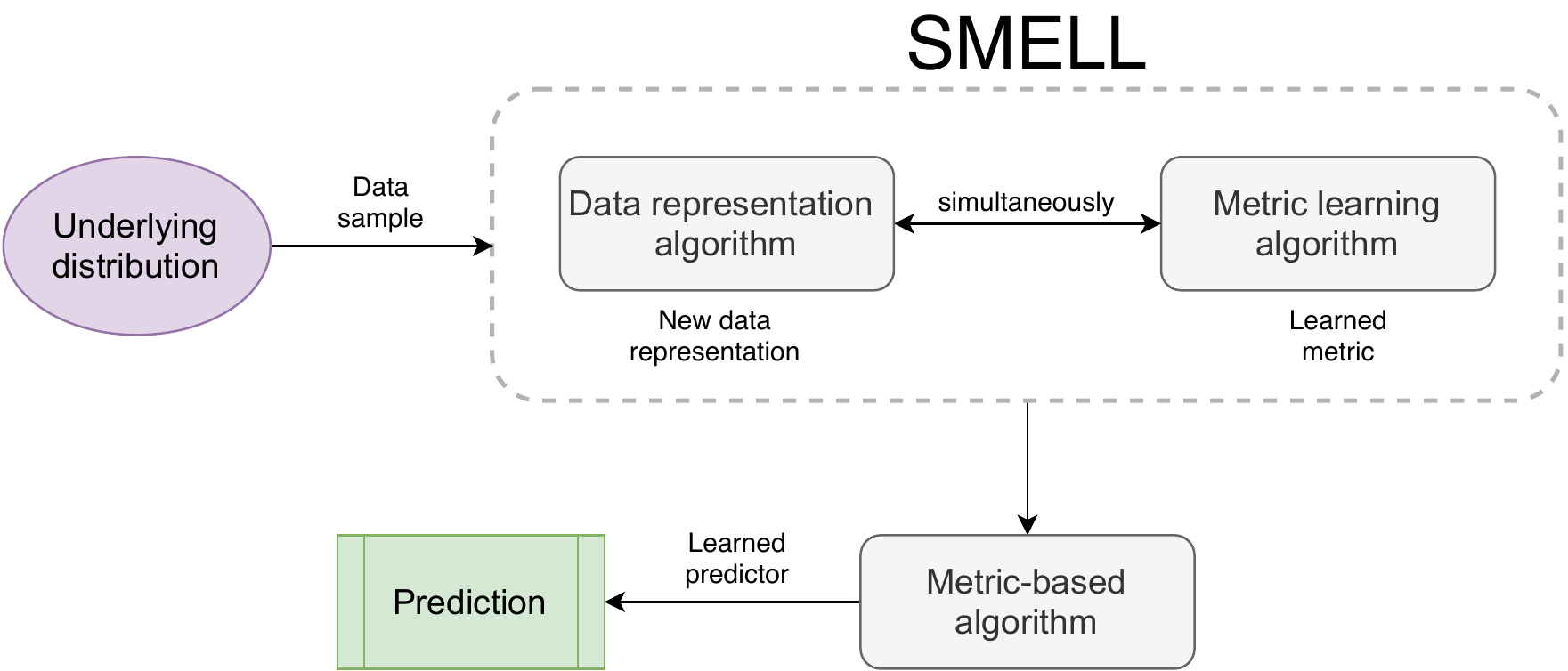}
\caption{Simple black box schematic for our proposal.}
\label{fig:sche_our}
\end{figure}

\subsection{Metric learning algorithm}
\label{subsubsec:prob_def}


There are several ways to find a similarity metric $\psi_\Sigma$~\cite{Wang2012, 7299016}. In this paper, we propose $\psi_\Sigma$ being estimated from the latent representation obtained by the encoder $f_\Theta: \mathcal{X} \xrightarrow{} \mathcal{Z}$.


\subsection{The S-space}
\label{subx4sec:sspace}



As can be seen in in Definition \ref{def:sspace}, we define a new representation space namely S-space $\mathcal{S}$, which quantifies the similarity between pairs of objects. In Equation~\ref{eq:elementwise}, we propose a map function $f^{S}: \mathcal{X} \times \mathcal{X} \rightarrow \mathcal{S}$ for a data pairwise $(\bm x_i,\bm x_j)$ as being an element-wise absolute value operation representing the pairwise difference between the pair of data.
Note, in Equation~\ref{eq:elementwise}, that $\bm s_{ij} \in \mathbb{R}^n$ (same dimension then \textit{latent representation space}).

Regarding the pairwise labeling, we have two options for a given pair $(\bm x_i, \bm x_j)$: similar or dissimilar.
Thus, we define the \textit{Markers set} $\mathcal{M} $  
so that each marker of $\mathcal{M}^+$ or $\mathcal{M}^-$ represents one of these possibilities (see Definition \ref{def:markersset}). The closer the vector $\bm s_{ij}$ is to a marker $\bm \mu^+ \in \mathcal{M}^+$ or $\bm \mu^- \in \mathcal{M}^-$, the greater the probability that the elements of the pair $(\bm x_i,\: \bm x_j)$ are similar or dissimilar to each other, respectively. Then, we have, in this case, $k$ similarity markers and $w-k$ dissimilarity markers for $\mathcal{M}$, and $\mathcal{M}^+ \cap  \mathcal{M}^- = \emptyset $.

Inspired by~\cite{maaten2008visualizing,  xie2016unsupervised, li2018discriminatively} we use the Student’s t-distribution  with one degree of freedom as a kernel to measure the similarity between  $\bm s_{ij}$ and a specific marker $\bm \mu_m \in \mathcal{M}$, as 
\begin{equation}
  q^m_{ij} = \frac{\left(1 + ||\bm s_{ij} - \bm \mu_m ||^2_2\right)^{-1}}{\sum_{\mu_{m'} \in \mathcal{M}} \left(1 + || \bm s_{ij} - \bm \mu_{m'} ||^2_2\right)^{-1}},
 \label{eq:eq1}
\end{equation}
where $q^m_{ij} \in \mathbb{R}$ is the similarity/dissimilarity of 
$\bm s_{ij}$ in relation to the markers $\mu_m$ (it is normalized by the sum of all markers in $\mathcal{M}$). So, we calculate $q^{+}_{ij} = \sum_{p} q^{p}_{ij}$ for all $ \bm \mu_{p} \in \mathcal{M}^+$ and  $q^{-}_{ij} = \sum_{n} q^{n}_{ij}$ for all $ \bm \mu_{n} \in \mathcal{M}^-$. In other words, $q^+_{ij}$ is the probability of $\bm x_i$ have the same label as $\bm x_j$ and $q^-_{ij}$ is the probability of $\bm x_i$ and $\bm x_j$ have different labels. Since $\mathcal{M}^+$ and $\mathcal{M}^-$ are two disjoint sets, we have $q^+_{ij} + q^-_{ij} = 1$.

It is worth noting that, we use a different version of the Deep Metric Learning canonical model.  Thus, we use the representation of the difference vector $\bm s_{ij}$ defined in S-pace. In Section~\ref{subsec:latent} we show more details about this choice. 


\subsection{Loss function and regularization}
\label{subsubsec:loss_fuction}


SMELL relies on simultaneously learning a latent representation of the data (with parameters $\Theta$ and $\Theta'$ for the encoder and decoder functions, respectively) and the positioning of the markers of the set $\mathcal{M}$ in S-space. Therefore, we seek to find the parameters $ \Sigma =  \{\Theta, \Theta', \mathcal{M}\}$ of the function $\psi_\Sigma(\bm x_i,\bm x_j)$ is defined as an optimization problem. Let the cost function be $J(\{\mathcal{X} \times \mathcal{X}\})$, we estimate the optimal parameters set $\Sigma^*$ with Cross-entropy loss $H_c$. We define regularization functions $R_r$ and $R_d$ to avoid overfitting in the training process. In training, the cross-entropy is applied between the output of SMELL and object's classes.

Similarly to~\cite{8237874}, $R_r$ regards to the autoencoder's reconstruction error. In our proposal, for all training pairs ($\bm x_i$, $\bm x_j$) and for all reconstructed pairs ($\bm x_i', \bm x_j'$) we have 
%
    $R_r = r_rN^{-1}\sum_i \sum_j \left(||\bm x_i - \bm x_i'||^2_2 + ||\bm x_j - \bm x_j'||^2_2\right)$,
\noindent where $r_r$ is a constant to calibrate the loss reconstruction function and $N$ is the number of pairs in train the dataset. 



When we use more than one maker as reference points to the similarity/ dissimilarity regions, markers of the same set $\mathcal{M}^+$ (or $\mathcal{M}^-$) tend to group altogether, hidering the efficiency of our method. In this context, we propose a new regularization term $R_d$ we called \textit{Repulsive Regularizer}, to avoid this undesirable behavior. It is defined  as

\begin{equation}
    R_d^+ =\frac{1}{c^+} \left[\sum_{\mu_i \in \bm{M}^+} \sum_{\mu_j \in \bm{M}^+}  \frac{1}{||\mu_i - \mu_j||^2_2+\epsilon}\right],
\label{eq:rd+}
\end{equation}

 \noindent where $\bm \mu_i \neq \bm \mu_j$ and $c^+$ is a constant value defined as $c^+ = {k \choose 2}$,
%
in which $k$ is the number of elements in $\mathcal{M}^+$ (see Definition \ref{def:markersset}). $R_d$ is inversely proportional to the square distance of the markers. To avoid ill-formed problems, we added to the  denominator a corrective term  $\epsilon$ that prevents division by 0. We conducted a manual investigation with a grid search, and we adopted for our experiments $\epsilon = 10^{-3}$. In the same way, we define $R_d^-$, and with that, we have
\begin{equation}
    R_d = r_d(R_d^+ + R_d^-),
\label{eq:rd}
\end{equation}
with a constant value $r_d$ for calibration. Note that $R_d^+ = 0$ if we have a single positive marker $k=1$. In the same way, if we have a single negative marker, $R_d^- = 0$ if $w - k = 1$. 




Let $\mathcal{Q} = \{q_{ij}\}$, the SMELL output, be the set that contains the pairs $ q_{ij}=(q^+_{ij},\: q^-_{ij})$ corresponding to the probability of the elements of a pairwise input $(\bm x_i, \bm x_j)$ be similar or dissimilar, respectively. The optimal hyperparameters set can be defined as $
\Sigma^* = 
\argminA_{\bm \Sigma} J( \{\mathcal{X} \times \mathcal{X}\})$, where 

\begin{equation}
J( \{\mathcal{X} \times \mathcal{X}\}) = H_c (\mathcal{U}||\mathcal{Q})r_{HC} + R_r + R_d,
\label{eq:loss}
\end{equation}
where $r_{HC}$ is a constant for calibration and $\bm u_{ij} \in \mathcal{U} $ is defined as $\bm u_{ij} = (1,0)$ if $i$ has same label as $j$ and $\bm u_{ij} = (0,1)$, otherwise.

\begin{figure*}
\centering
\includegraphics[width=1\linewidth]{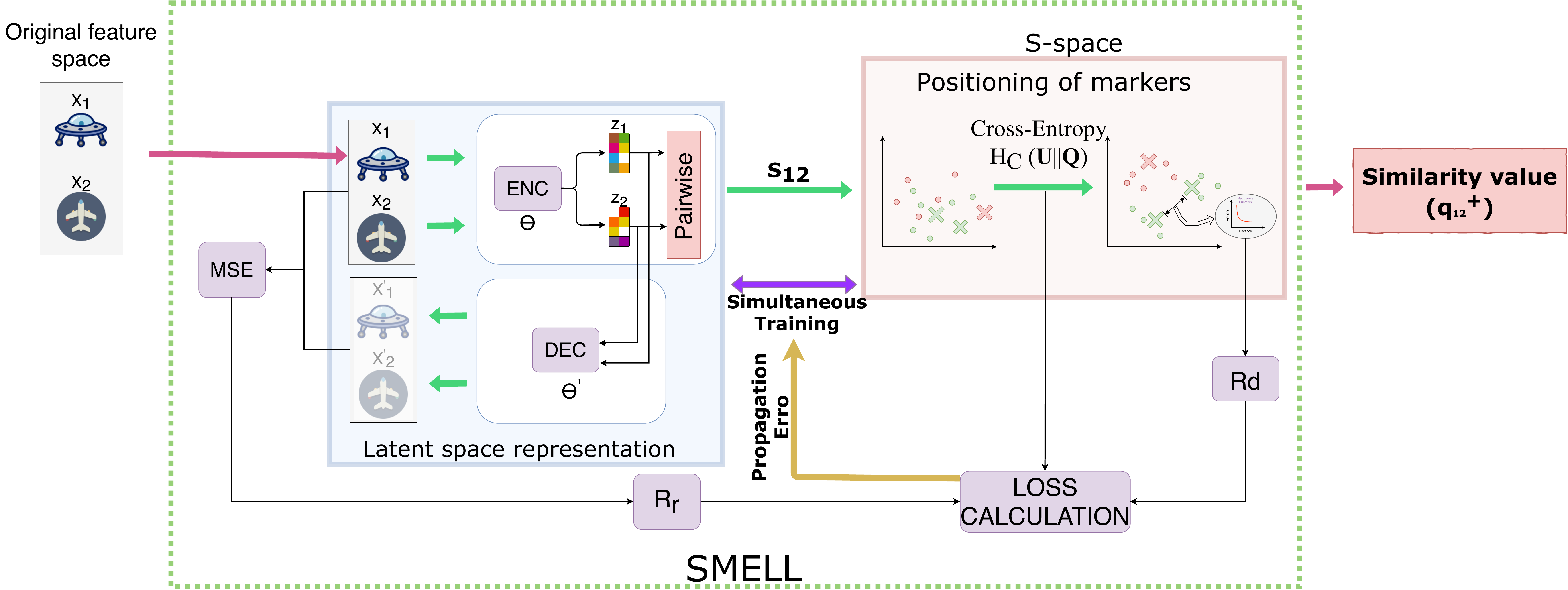}
\caption{The left side represents the Encoder with reconstruction, and the right side represents the optimization process for the markers' position for $\mathcal{M} = \{ \bm \mu_1^+,\: \bm \mu_2^+,\: \bm \mu_3^-\}$. In this example, we used two positives and one negative marker. Green and red crosses represent $ \bm \mu_1^+$, $\bm \mu_2^+$, and $\bm \mu_3^-$, respectively, green and red dots represent the similar and dissimilar input pairs. The rightmost green arrow shows a representation of the markers' position optimization step by using  Cross-Entropy divergence and some regularization functions. 
Observe that the number of positive and negative markers are hyperparameters.}
\label{fig:sche_siam_full}
\end{figure*}

%

SMELL learns all parameters in the set $\Sigma^*$ simultaneously. The representation found in S-space aims at grouping the elements $\bm s_{ij} $ around their respective markers, as defined in Loss Function $J$ (Equation~\ref{eq:loss}). The impact of the attractive behavior is controlled by the constant $r_{HC}$, i.e., the higher the $r_{HC}$, the greater is the tendency to group the points $\bm s_{ij}$ closer to the respective markers. Also, note that the regularization functions operate in different spaces, i.e. $ R_r $ operate in \textit{latent feature space}, $ R_d $ operates in \textit{S-space} and $ H_c $ operates in \textit{latent feature space} and~\textit{S-space} simultaneously.

 Figure~\ref{fig:sche_siam_full} depicts the more detailed schematic of our proposal using a toy example (two positive markers and one negative). Observe that the number of positive and negative markers is a hyperparameter.

\subsection{Optimization}

To find the $\Sigma^*$ set, we use mini-batch stochastic gradient decent (SGD) and backpropagation. First, we note that the decoder weights $\Theta'$ are only affected by the $R_r$ component of the loss function $J$. So, we can use $\partial R_r / \partial  \Theta'$ to update $\Theta'$. Then, given a mini-batch with $g$ samples and learning rate $\lambda$, $\Theta'$ is updated by
\begin{equation}
\Theta' = \Theta' - \frac{\lambda}{g}\sum_{i = 1}^g\frac{\partial R_r}{\partial \Theta'}.
\label{eq:opttheta}
\end{equation}

To optimize the markers, consider that 
\begin{equation}
\bm \mu_{t} = \bm \mu_{t} - \frac{\lambda}{g}\sum_{i = 1}^g\frac{\partial J}{\partial \bm \mu_{t}} = \bm \mu_{t} - \frac{\lambda}{g}\sum_{i = 1}^g\bigg(\frac{\partial L_{HC}}{\partial \bm \mu_{t}} + \frac{\partial R_{d}}{\partial \bm \mu_{t}}\bigg),
\label{eq:optmarkers}
\end{equation}

\noindent where $\dfrac{\partial L_{HC}}{\partial \bm \mu_t}$ can be calculated for a given $\bm \mu_t$ and $\bm s_{ij}$ as

$$
\frac{\partial L_{HC}}{\partial \bm \mu_t} = 2 \frac{(q_{ij}^t- \bm u_{ij})(\bm s_{ij} - \bm \mu_t)}{1 + ||\bm s_{ij} -\bm \mu_t||_2^2},
$$
 and
 $$
\frac{\partial R_d}{\partial \bm \mu_t} = -2\sum_{\bm \mu_s \in M} \bigg[\text{sign}(\bm \mu_s)\frac{||\bm \mu_t - \bm \mu_s||_2}{(||\bm \mu_t - \bm \mu_s||_2^2 + \epsilon)^2}\bigg],
$$
where sign($\bm \mu_s$) = 1 if $\bm \mu_s \neq \bm \mu_t$ and $\bm \mu_s$ has same semantic (similarity or dissimilarity) than $\bm \mu_t$, and sign($\bm \mu_s$) = 0, otherwise. 

For training SMELL, we randomly selected the mini-batch with $ m $ pairs of elements (half are similar, and the other half are dissimilar). Also, our proposal does not have any specific batch selection criteria.

\subsection{Theoretical proprieties}
\label{subsec:theo}

Due to the construction of the S-space, we are able to obtain some theoretical proprieties. 

\begin{definition}\label{ax:ex}


({\it Optimal Latent Space}) Let $\bm x_i, \bm x_j \in \mathcal{X} $ and a latent representation function $f_\Theta: \mathcal{X} \rightarrow \mathcal{Z}$. The transformation $ f_\Theta $ generates an {\it optimal latent space} $\mathcal{Z}$ when the expected value $\mathbb {E}[||\bm s_{ij}||_2] = 0 \Longrightarrow {} l (\bm x_i) = l (\bm x_j).$
\end{definition}

SMELL is able to group points of same class into clusters. It is worth noting that we defined the optimal space as a conditional instead of a biconditional statement. From this definition, we can observe that SMELL may create several different clusters of the same class, as depicted in Figure~\ref{fig:sche_k_optimal}.

\begin{figure}
\centering
\includegraphics[width=0.3\linewidth]{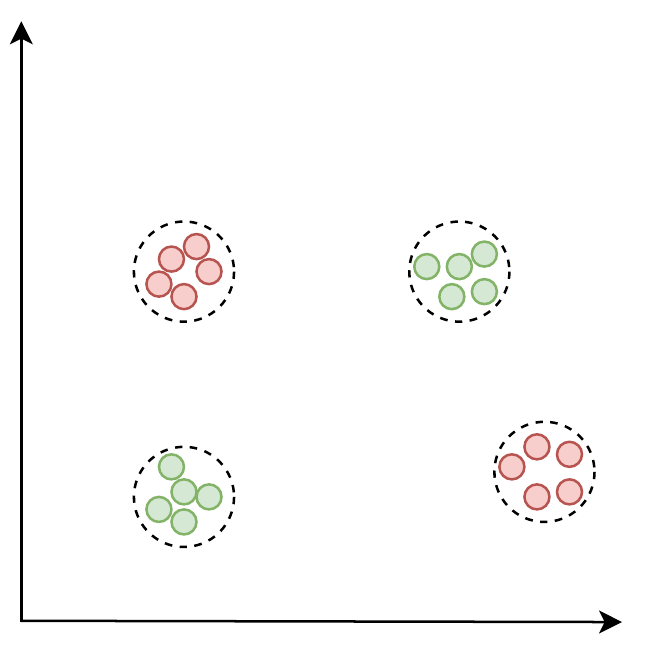}
\caption{Toy example for \textit{optimal latent space}. Same color indicates same label. S-space requires that each group has only elements of same class, but note that, this space can have different groupings with elements of the same class.}
\label{fig:sche_k_optimal}
\end{figure}

\begin{theo}
\label{pro:mod}

In S-space, given $k$ positive markers in the set $ \mathcal{M}^+  $ and $n-k$ negative markers in  $  \mathcal{M}^-$, the latent space found by SMELL, i.e., the estimation of the parameters $\Theta$ of  $f_\Theta$, generates an {\it optimal latent space} if $\exists\; \bm \mu_i \in \mathcal{M}^+$ so that $||\mu_i||_2^2 <  ||\mu_j||_2^2$ for any $\bm \mu_j \in \mathcal{M}^-$.

\end{theo}

\begin{proof}
The proof for this proposition can be found in~\ref{sec:appendix}.
\end{proof}

From Proposition~\ref{pro:mod}, if SMELL finds a \textit{optimal latent space}, at least one positive marker has a smaller norm than the negative marker. In addition, in practice, as we can see in the Section~\ref{sec:discussion}, at least one positive marker is smaller than all negatives markers (the positive marker is located near the origin). We observed that the model builds a latent space of groups with  elements of the same class, similar to the Figure~\ref{fig:sche_k_optimal}.

\begin{theo}
\label{theo:risk}
For S-spaces built with one marker in each group, $\mu^+ \in M^+$ and $\mu^-\in M^-$, $D^-$ and $D^+$ being the Euclidean distance of an object to the negative and positive marker, respectively, the misclassification risk function of a positive marker is \begin{multline*}
R^+ = \frac{(D^-)^2+1}{\sqrt{(D^-)^2+2}}\Bigg[\left(\sqrt{(D^-)^2+2}\right)log\left(\frac{1}{\sqrt{\frac{(D^+)^2}{(D^-)^2+2}+1}}\right) -\\- \frac{((D^-)^2+1)\left[tan^{-1}\left( \frac{(D^+)}{\sqrt{(D^-)^2+2}} \right)\right]^2}{\sqrt{(D^-)^2+2}} +\\+ (D^+)tan^{-1}\left(\frac{(D^+)^2}{\sqrt{(D^-)^2+2}}\right)\Bigg].
\end{multline*}.

\end{theo}

\begin{proof}
The proof can be found in~\ref{sec:appendix}.
\end{proof}

Due to the S-space formulation, we obtain the probability of a pair being similar analytically, given the distance of that pair to the positive marker (typically this probability is estimated, as we can see in~\cite{hist_loss}).

\section{Experimental setup}
\label{sec:experiments}

We conducted an extensive set of experiments in several scenarios with diffe\-rent setups to understand SMELL behavior and effectiveness better. 
Section~\ref{subsec:database} describes the datasets we have employed. 
Section~\ref{subsec:evaluation_def} details the classification protocol designed to evaluate our method and the baselines.
Section~\ref{subsec:initialization} discusses the initialization and the architecture of the proposed approach.

\subsection{Dataset}
\label{subsec:database}

\subsubsection{General purpose datasets}

KEEL~\cite{keel} is an open source\footnote{\url{http://keel.es/datasets.php}} Java software tool that can be used for a large number of different knowledge data discovery tasks. We used 28 datasets provided by KEEL to evaluate our proposal. All datasets are numeric and have no elements missing. Furthermore, all datasets have been min-max normalized to the interval $[0, 1]$, a precondition to the experiments' execution.

There is a wide variety of data in KEEL. The 28 datasets used in our experiments are divided into Medical data (Bupa, Cleveland, Appendicitis, Newthyroid, Pima, Wdbc, Wisconsin, and Thyroid); natural Language Processing data (Vowel, Letter, and Phoneme); experimental psychological data (Balance); feature-based image data (Magic, and Satimage); hierarchical decision-making data (Monk-2, Ring, and Twonorm); nature data (Iris and Banana); disaster prediction data (Titanic); Weather data (Ionosphere); chemical data (Glass, Wine, Winequality-red); and Object/shape recognition (Sonar, Movement\_libras, and  Vehicle).

These 28 datasets have substantial diversity in terms of data factors: the number of examples, the number of features, and the number of classes. Specifically, the number of examples ranges from 106 to 2003, and the number of features ranges from 2 to 90. The datasets contain both binary and multiple class datasets with a maximum of 26 classes for one dataset. 

Although our method scales up to large datasets, some methods do not; hence, due to a large number of datasets, we downsampled some of them (the ones with more than 1000 samples) to 10\% of the original size. 
The characteristics of datasets are described in  Table~\ref{tab:db}.

\begin{table}
\centering
\small
\begin{tabular}{cccc}
\hline
Dataset & \#Examples & \#Features & \#Classes \\ \hline
\rowcolor{mygray} Appendicitis & 106 & 7 & 2 \\ 
Balance        &    625           &    4            &  3             \\ 
\rowcolor{mygray}Banana (10\%)      &    530           &    2            &  2             \\ 
Bupa & 345 & 6 & 2 \\ 
\rowcolor{mygray}Cleveland & 297 & 13 & 5 \\ 
Glass        &    214           &    9            &  7             \\
\rowcolor{mygray} Ionosphere & 351 & 33 & 2 \\ 
Iris    & 150           & 4              & 3             \\
\rowcolor{mygray} Letter (10\%) & 2003 & 16 & 26 \\
Magic (10\%)      &    1902           &    10            &  2             \\ 
\rowcolor{mygray} Monk-2       &    432           &    6            &  2             \\ 
Movement-libras       &    360           &    90            &  15             \\ 
\rowcolor{mygray} Newthyroid & 215 & 5 & 3 \\ 
Phoneme (10\%)      &   541        &       5         &          2     \\ 
\rowcolor{mygray} Pima & 768 & 8 & 2 \\  
Ring (10\%) & 740 & 20 & 2 \\
\rowcolor{mygray} Satimage (10\%) & 643 & 36 & 7 \\
Segment (10\%)      &  231         &       19         &        7       \\ 
\rowcolor{mygray} Sonar       &    208           &    60            &  2             \\ 
 Thyroid (10\%) & 720 & 21 & 3 \\
\rowcolor{mygray}Titanic (10\%) & 221 & 3 & 2 \\
Twonorm (10\%) & 683 & 20 & 2 \\
\rowcolor{mygray} Vehicle       &    846           &    18            &  4             \\ 
Vowel & 990 & 13 & 11 \\
\rowcolor{mygray} Wdbc       &    569           &    30            &  2             \\ 
Wine & 176 & 13 & 3 \\ 
\rowcolor{mygray}Winequal-red (10\%) & 160 & 11 & 11 \\
 Wisconsin & 683 & 9 & 2 \\
\hline
\end{tabular}
\caption{Datasets used in our experimets.}
\label{tab:db}
\end{table}

\subsubsection{The MNIST dataset}


The MNIST dataset\footnote{\url{http://yann.lecun.com/exdb/mnist/}} is one of the most common datasets used for image classification and accessible from many different sources. The data set consists of grayscale images with 28x28 dimensions. Following~\cite{drlim}, the training set is built from all hand-written digits 4 and 9 from the MNIST dataset.

Due to a large number of MNIST features, the spatial correlation found in the images, and a large number of samples, we consider the dataset suitable for this evaluation. All images were normalized to the interval $[0,1]$, resulting in 6958 and 6824 images corresponding to hand-written digits 4 and 9, respectively.

\subsection{Network evaluation}
\label{subsec:evaluation_def}

To evaluate all metric learning techniques assessed in this work, including our approach, we apply a \textit{K-Nearest Neighbor} (KNN) classifier, with three neighbors, in agreement with~\cite{Deng:2016:EKC:2937259.2937570}. The KNN classification performance can often be significantly improved through (supervised) metric learning. In this work, the KNN classification can be exchanged for any other algorithm that uses a metric.

Since we used several datasets to validate our proposal, we divided our assessment into two approaches. The first is an individual evaluation for each dataset, and the second is a general evaluation for all datasets. 


For each dataset, we calculate the accuracy. For all datasets (except the MNIST), we calculate the average accuracy (Accuracy\_AVG), the average rank position value (Ranking\_AVG), and the difference of the accuracy average for the best proposal (Diff\_AVG). We also calculated the number of times our algorithm was in the first position (\# of 1$^{\text{st}}$). For the MNIST dataset, we evaluated the proposals' accuracy for different \textit{latent space representation} dimensions.

We used 10-fold cross-validation. This validation can largely retain heterogeneous distributions in the training set and improve statistical confidence in the results. For the sake of reproducibility, our proposal is publicly available on a Gitlab repository\footnote{https://gitlab.com/sufex00/smell}.

 \subsection{Parameters initialization and network architecture}
 \label{subsec:initialization}

We initialize all weights of the autoencoder layers from a zero-mean normal distribution $\mathcal N(\mu=0, \sigma=0.01)$. Biases were also initialized as outcomes of a normal distribution $\mathcal N(\mu=0.5, \sigma=0.01)$, following~\cite{koch2015siamese}. Markers position are initialized with Lloyd's algorithm~\cite{lloyd}.
Furthermore, we pre-trained an autoencoder (without markers) and further transfer the learn to the complete model (with markers) to improve the convergence speed.
%

The encoder of all deep metric learning approaches used as baseline is identical to the one we used in SMELL. According to~\cite{xie2016unsupervised}, we set network dimensions to $m$-512-512-2048-$n$ for all datasets, where $m$ is the number of features of the input data, and $n$ is the \textit{latent space representation} dimension. All layers are fully connected, and we used as activation function the Rectified Linear Unit (ReLU)~\cite{Nair:2010:RLU:3104322.3104425}. 

In addition, we used mini-batch Stochastic Gradient Descent (SGD) where learning rate is $0.01$ and momentum is $0.9$. All parameters previously mentioned (except for the calibration of the markers) were used in all deep metric learning baseline and our proposal. 

Since the optimization model depends on some hyperparameters $(r_{HC}$,  $r_d$, $r_r$, $w$, $k)$, we performed an investigation to determine which value of these variables would maximize the model accuracy. Therefore, we randomly chose the Vehicle dataset to train the model and  select the hyperparameters.

In \cite{MOCKUS1975428}, the author proposed a method called Bayesian Optimization, which consists of optimizing functions such as a ``black box''.
The method consists of, with some known points, determining the shape of the function by regression. Usually, this prediction is made through a Gaussian process due to some characteristics (scalable to a few points and not parametric). Based on the regression of the Gaussian process, it is defined a utility function that consists of finding the next candidate for the parameters aiming at the optimization of some specific metric. 

Because some hyperparameters are defined in a discrete interval, such as the number of markers, it was necessary to perform the discretization of the Bayesian Optimization values. We used five random starting points, and then 20 rounds of the algorithm, where we found $r_{HC} = 1$ , $r_d = 10^{-1}$, $r_r = 10^{-3}$, similarity markers $k = 3$, dissimilarity markers $w-k = 2$, these values were used in the rest of this work. We realized that our proposal typically performs well when $ k $ has a value similar to $ w-k $.

We configured all baselines with the hyperparameters recommended in their original articles.These parameters are listed below. Observe that three approaches (Euclidean, NCA and NPair) do not have hyperparameters to set.

\begin{itemize}
    \item {Metric Learning algorithm}

\begin{list}{--}{}

    \item ANMM~\cite{anmm}: The size of the homogeneous and heterogeneous neighborhoods for each data point is set to 10; 
    
    \item KDMLMJ~\cite{NGUYEN2017215}: Let $k_1$, $k_2$ denote the number of neighbors for constructing the positive and negative difference spaces, we used $k_1=k_2 = 5$.

\end{list}

\item {Deep Metric Learning algorithm}

\begin{list}{--}{}
    \item ContrastiveLoss~\cite{drlim}: The margin term equals 1;
    \item Triplet~\cite{Schroff_2015_CVPR}: The margin term equals 0.2;
    \item MultiSimilarityLoss~\cite{Wang_2019_CVPR}: The hyper-parameters for the model are : $\alpha= 2$, $\lambda= 1$, $\beta= 50$;
    \item FastAPLoss~\cite{Fastaploss}: The number of soft histogram bins for calculating average precision is 10.
\end{list}
\end{itemize}
\section{Results and Discussion}
\label{sec:discussion}


In this section, we present the results of the SMELL's assessment. We also discuss the interpretability of the similarity space (S-space) and conduct a performance evaluation comparing SMELL with three distance metric learning approaches from pyDML\footnote{https://pydml.readthedocs.io/en/latest/index.html}~\cite{anmm, NIPS2004_2566, NGUYEN2017215}, five deep metric learning approaches~\cite{drlim, wang2019multi, Schroff_2015_CVPR, Nparloss, Fastaploss}, and Euclidean distance.
\subsection{Ablation Study}
\label{sec:ablation}

For a better understanding of our proposal, we conducted an ablation study. Therefore, we evaluated SMELL for different regularization calibration values. We evaluated SMELL with and without the reconstruction error ($ r_r = 0 $), with and without the repulsive error ($ r_d = 0 $), and without both ($ r_d = r_r = 0 $). The other default values adopted in our experiments, can be found in Section~\ref{sec:experiments}.

Besides, we also evaluated the behavior of our proposal when using S-space only for training. We then use for prediction a version of SMELL without the S-space (using Euclidean distance), we named this approach SMELL (Euclidean). 
Therefore, after training the model using S-space, we observe only the latent space to perform the similarity metric's extraction, i.e., we consider that the similarity between two objects is the Euclidean distance between them in the latent space. It is also worth noting that we use the same default values adopted in our experiments (without any restriction on $r_d$ and $r_r$). A summary of results is in Table~\ref{tab:ablation}. We also provide a complete report of our results in Table~\ref{tab:ablation_full} (\ref{sec:appendix_ablation}).

We see that among the usual SMELL methods when we take $ r_d = 0 $, the proposal tends to have performance degradation. This behavior is easily seen in the dataset ring (see Table~\ref{tab:ablation_full} in \ref{sec:appendix_ablation}), in which SMELL ($ r_d = 0 $) and SMELL ($ r_d = r_r = 0 $ ) has an accuracy of 0.6536 and 0.6610, respectively. Comparing this value with the best result, we have a difference of more than 20\%. The proposal with $ r_d = 0 $ has the worst performance in all four metrics analyzed (excluding Euclidean).

\begin{table}[]
\centering
\small
\begin{tabular}{p{3.4cm}
>{\columncolor{mygray}}c c
>{\columncolor{mygray}}c c}
\hline
    Propose & Accuracy\_AVG  & Ranking\_AVG  & Diff\_AVG & \# of 1$^{\text{st}}$
     \\

     \hline
     
SMELL ($r_r=0$) & 0.8254 &  \textbf{2.2857} & 0.0116 & 9
\\
SMELL ($r_d=0$) & 0.8169     &   2.8571 & 0.0201 & 5 
\\
SMELL ($r_r=r_d=0$) & 0.8178 &   \textbf{2.2857} & 0.0193  & 9 
\\
SMELL (Euclidian) & 0.7994 &   3.6429 & 0.0376 & 4
\\
SMELL (S-space) & \textbf{0.8268} &   \textbf{2.2857}  & \textbf{0.0102} & \textbf{10}
\\
\hline
\end{tabular}
\caption{Performance comparison summary of ablation study when using  KNN classification for 27 different datasets. The full table can be see in~\ref{sec:appendix_ablation}}.
\label{tab:ablation}
\end{table}

When $ r_r = 0 $, we observe a slight impact on the result (when compared to $ r_d = 0 $), but for the datasets Appendicitis, Vowel, Banana (10\%) and Twonorm (10\%) (see Table~\ref{tab:ablation_full} in \ref{sec:appendix_ablation}), changing $ r_r $ to zero, made SMELL stop being the first position (when analyzing accuracy), to the second last position.

When we consider the case of SMELL with the Euclidean metric (instead of S-space), our proposal has the worst performance among the cases analyzed for the four metrics adopted. In particular, we see that the average difference for the first place (DIFF\_AVG) has increased 200\%. In addition, it is worth noting that for the datasets Twonorm (10\%), Banana (10\%), Wdbc, Movement\_libras and Appendicites (see Table~\ref{tab:ablation_full} in \ref{sec:appendix_ablation}), SMELL goes from first place to last place. This evidencing the limitation of the Euclidean metric (even using the function $ f_\Theta $ found by our proposal).

\begin{figure*}
\centering
\begin{subfigure}{.45\textwidth}
  \centering
  \includegraphics[width=1.0\linewidth]{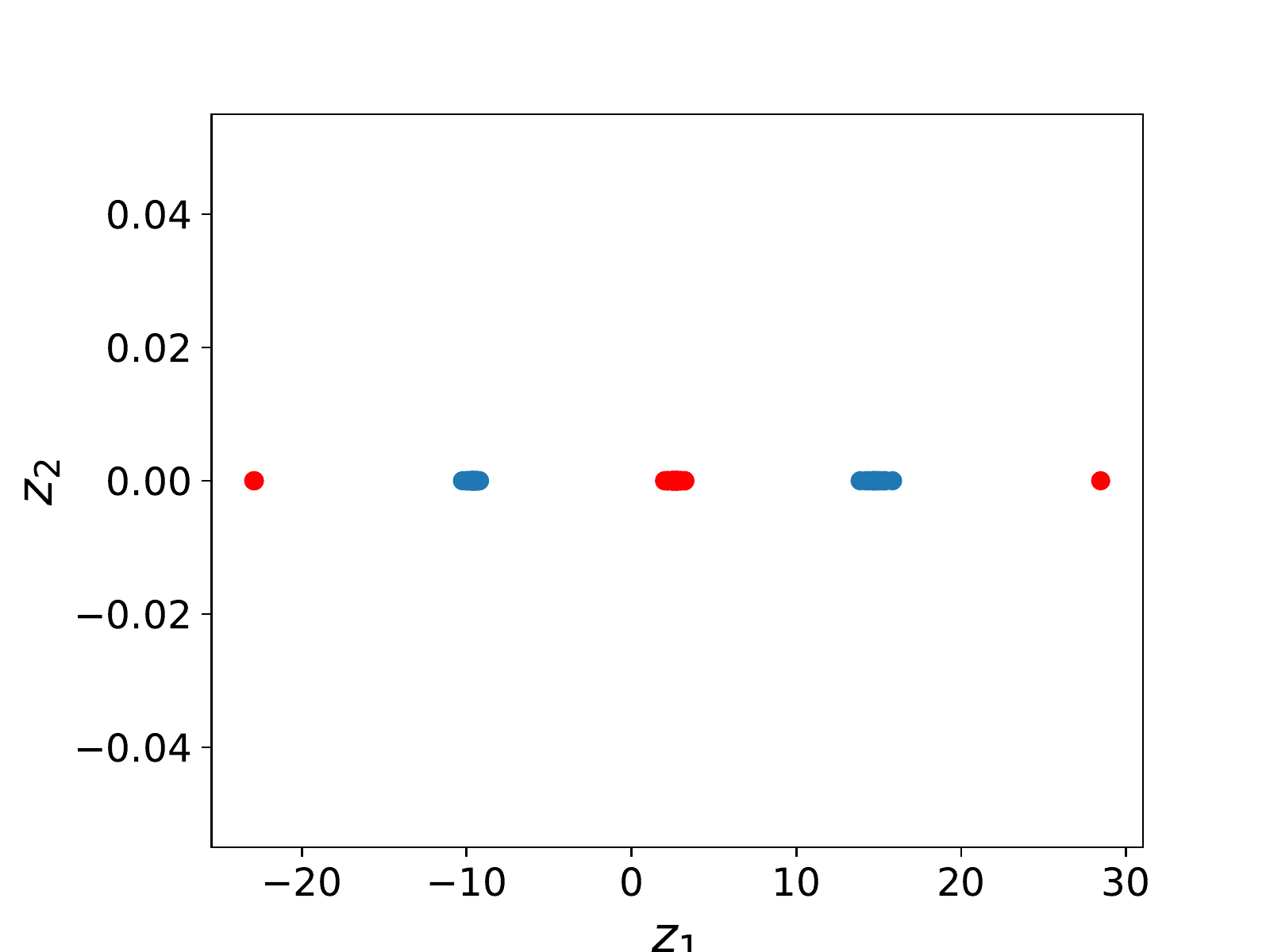}
  \caption{Latent space ($n=2$) from SMELL.}
  \label{fig:encoder_sonar_smell}
\end{subfigure}
\begin{subfigure}{.45\textwidth}
  \centering
  \includegraphics[width=1.0\linewidth]{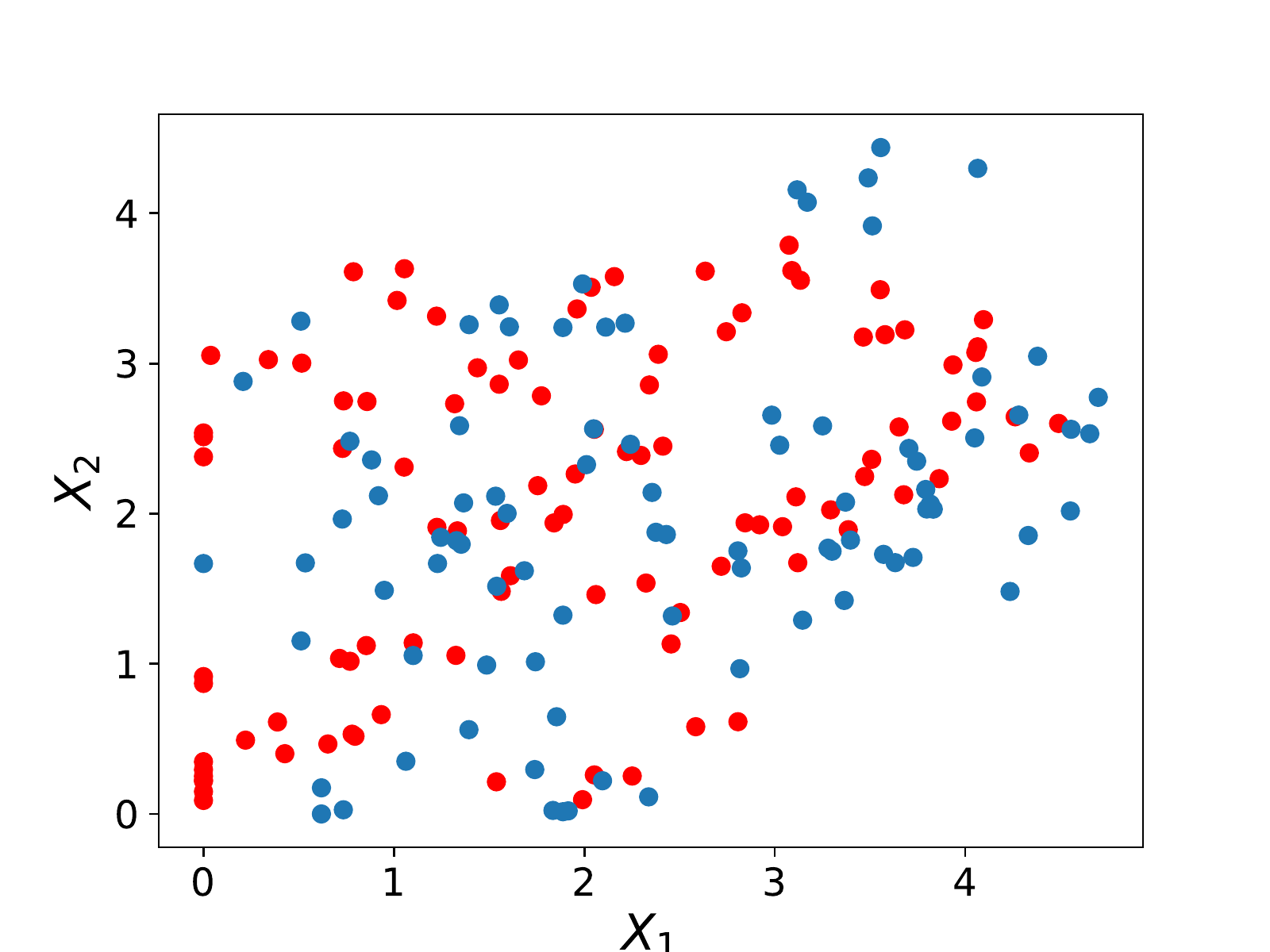}
  \caption{Encoder output ($n=2$) without markers.}
  \label{fig:encoder_sonar_autoencoder}
\end{subfigure}%
\caption{Sonar Dataset: Latent Space and Encoder space Analysis.}
\label{fig:encoder_sonar}
\end{figure*}



We hypothesized that SMELL tends to find a representation in S-space that captures similarity semantics. 
SMELL (S-space) has the best performance among all the versions used, evidence of the last statement. In addition, Repulsive regularizer tends to increase the separability of latent space. This fact shows the importance of the repulsive regularizer.

Moreover, we evaluated the behavior of SMELL in comparison with an autoencoder (without markers). Figure~\ref{fig:encoder_sonar} shows the behavior of the encoder output under SMELL and the autoencoder. In Figure~\ref{fig:encoder_sonar_smell}, whose encoder was used with SMELL, the classes have well-defined groups, differently to Figure~\ref{fig:encoder_sonar_autoencoder}, where there is a greater dispersion of the classes, with no clustering pattern being observed. The same behavior is found in the MNIST dataset, as shown in  Figure~\ref{fig:encoder_mnist}.

\begin{figure*}
\centering
\begin{subfigure}{.45\textwidth}
  \centering
  \includegraphics[width=1.0\linewidth]{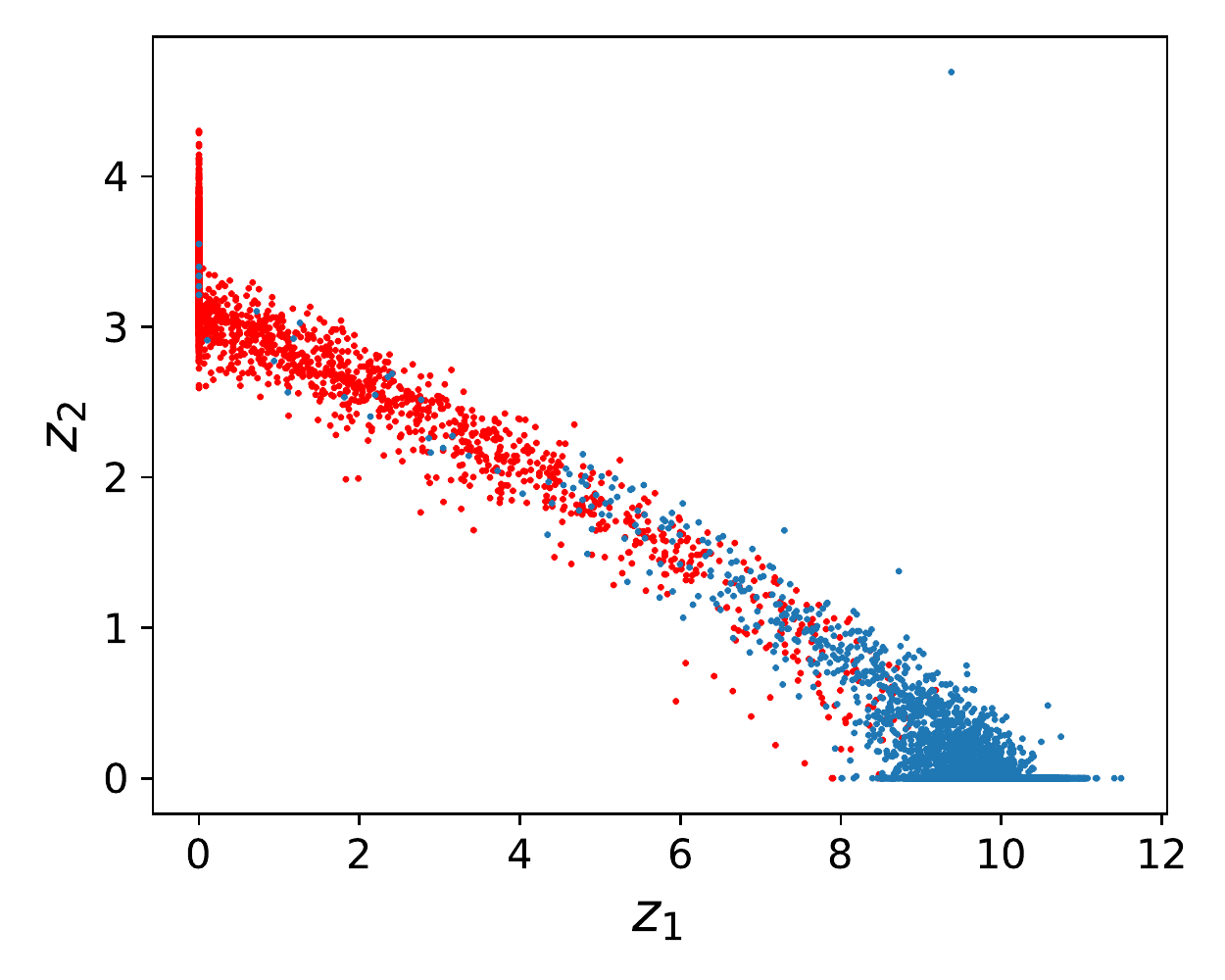}
  \caption{Latent space ($n=2$) from SMELL.}
  \label{fig:encoder_mnist_smell}
\end{subfigure}
\begin{subfigure}{.45\textwidth}
  \centering
  \includegraphics[width=1.0\linewidth]{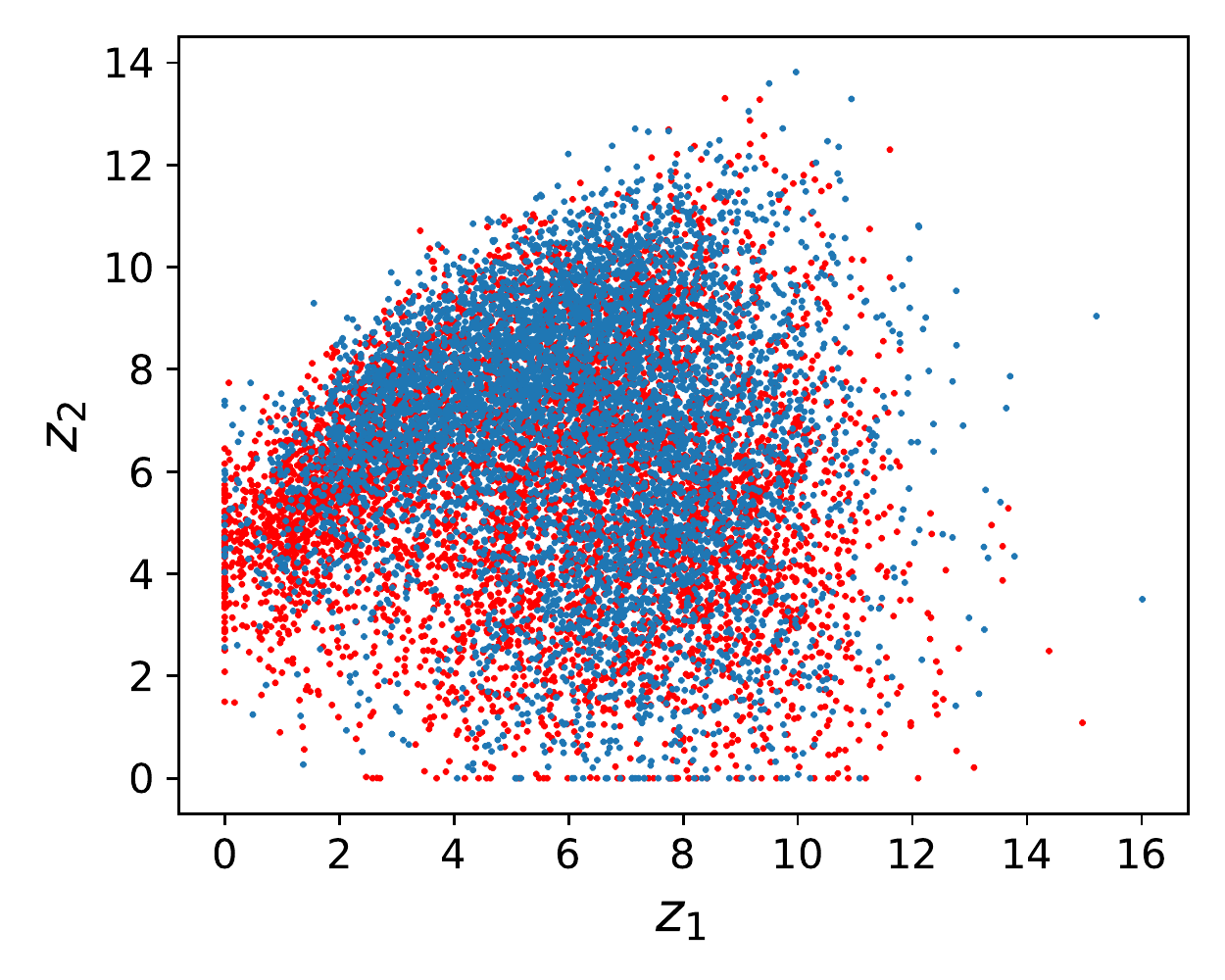}
  \caption{Encoder output ($n=2$) without markers.}
  \label{fig:encoder_mnist_autoencoder}
\end{subfigure}%
\caption{MNIST Dataset: Latent Space and Encoder space Analysis.}
\label{fig:encoder_mnist}
\end{figure*}




\subsection{Performance Comparison}
\label{subsec:comparison}

\begin{table}
\centering
\small
\begin{tabular}{p{2cm}
>{\columncolor{mygray}}c c
>{\columncolor{mygray}}c c
>{\columncolor{mygray}}c}
\hline
    Dataset &
    Category &
    Accuracy\_AVG &
    Ranking\_AVG   &
    Diff\_AVG &
    \# of 1$^{\text{st}}$
     \\ \hline
     
ANMM\cite{anmm} & MeL & 0.7768 & 5.7143 & 0.0726 & 0
\\
KDMLMJ\cite{NGUYEN2017215}  & MeL &0.7824 & 4.9643 & 0.0667    &   5
\\
Contrastive\cite{drlim}  & DMeL&0.6923 & 7.1786 & 0.1571 & 1
\\
MSLoss\cite{wang2019multi} & DMeL &0.8081 & 3.7857 & 0.0418 &  5
\\
Triplet\cite{Schroff_2015_CVPR} & DMeL & 0.8115 & 3.9643 & 0.0379 & 4
\\
NCA\cite{NIPS2004_2566} & MeL &0.7732 & 5.6786 & 0.0762 & 3
\\
NPair\cite{Nparloss} & DMeL &0.7380 & 6.2143 & 0.1114 & 2
\\
FastAP\cite{Fastaploss} & DMeL & 0.7801 & 4.7857 & 0.0693 & 3
\\
Euclidean & - &0.7486 & 5.8929 & 0.1007 & 1
\\
SMELL & DMeL & \textbf{0.8268} & \textbf{3.6429} & \textbf{0.0226} & \textbf{7}
\\
\hline
\end{tabular}
\caption{Performance comparison of some distance metrics approaches and SMELL when using  KNN classification for 27 different datasets. The best results are in \textbf{bold}.}
\label{tab:results_p1}
\end{table}

To compare our results to other techniques present in the literature, we used the datasets and the metrics appointed in Section~\ref{sec:experiments}, with $n=64$ (latent dimension). 
The summary of results can be found in Table~\ref{tab:results_p1} (for complete results, see~\ref{app_results}). The second column indicates whether the approach is based on Metric Learning (MeL) or Deep Metric Learning (DMeL) techniques. We compare SMELL to metric learning approaches~\cite{anmm, NIPS2004_2566, NGUYEN2017215}, deep metric learning approaches~\cite{ drlim, wang2019multi, Schroff_2015_CVPR, Nparloss, Fastaploss}, and the usual Euclidean distance. The k-fold cross-validation results are shown by averaging the standard deviation and accuracy values reported by the process. 

SMELL achieved the best accuracy results in 7 (\# of 1$^{\text{st}}$) datasets, thus surpassing all other analyzed algorithms (improving 40\% more datasets when compared with second best). KDMLMJ and MSLoss, the second-best, achieved the best result in 5 datasets. SMELL achieved an accuracy of 0.8268 (Accuracy\_AVG).  The second-best, Triplet, achieves 0.8115, and the third-best, MSLoss, achieves 0.8081. In a simple dataset (Monk-2), SMELL achieves 100\%. It is worth mentioning that SMELL, even in some situations its performance is not the best, reaches accuracy close to the best algorithm. For instance, the average distance between SMELL and the best algorithm is 2.26\% (Diff\_AVG), improving its average distance by 67.70\% and 84.96\% compared to Triplet (second-best) and MSLoss (third-best), respectively. Finally, when we average the ranking, SMELL achieved an average of $3.6429$ (Ranking\_AVG), the smallest value among all algorithms. The second-best was MSLoss, reaching $3.7857$. 

In Figure~\ref{fig:compare}, we compare SMELL's accuracy with all othe approaches used in this paper. We noticed that SMELL, in all cases, manages to overcome the techniques presented when we compare the number of individual hits, i.e., the number of datasets that SMELL exceeds the accuracy of the analyzed baseline. Besides, we noticed a small scattering of the blue dots around the black line compared to the red triangles' behavior. It indicates that even when SMELL performs worst than another approach, its results are close to the best.

 \begin{figure}
    \centering
    \includegraphics[width=1.0\linewidth]{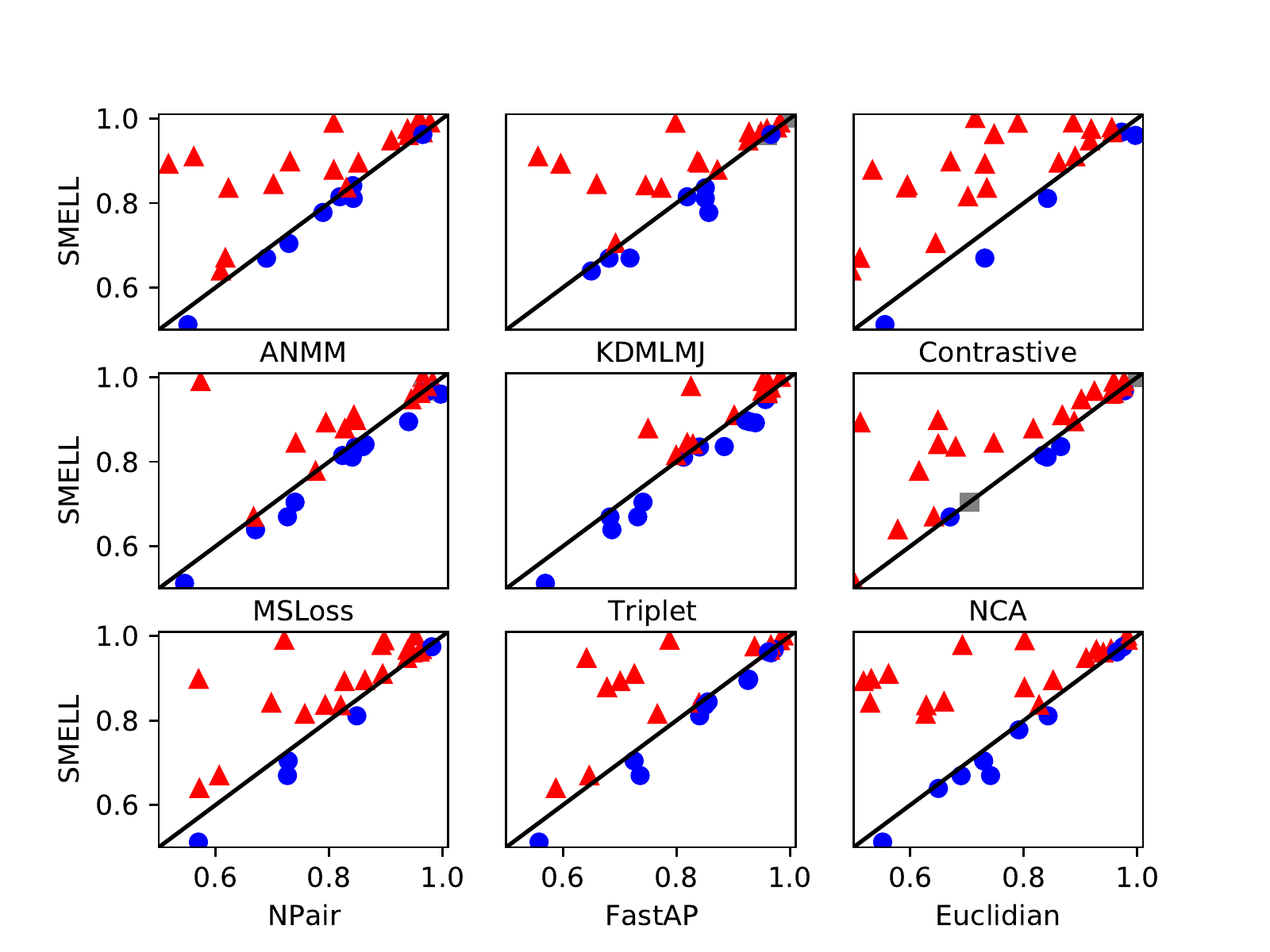}
    \caption{Comparison between SMELL and all proposals used as a baseline. The red triangles, blue dots and gray squares indicate when the SMELL is superior, inferior and has the same mean accuracy result, respectively, when compared to other proposals.}
    \label{fig:compare}
\end{figure}

Analyzing the metric learning approaches only (see Table~\ref{tab:results_p1}), we see that KDMLMJ and NCA algorithms achieve better results (among the algorithms adopted as baseline) when considering the metric that counts the number of times that the algorithm's accuracy surpassed all the others. This behavior is because the algorithms have been evaluated with KNN, and these algorithms were specifically designed to improve this classifier. 


Considering the MNIST data, we can see that our proposals achieves considerably better results, particularly for lower dimensions ($d = \{1,2,4\}$). This characteristic is highlighted by the area under the curve, as seen in Figure~\ref{fig:mnist_acc}.
 \begin{figure}
    \centering
    \includegraphics[width=0.90\linewidth]{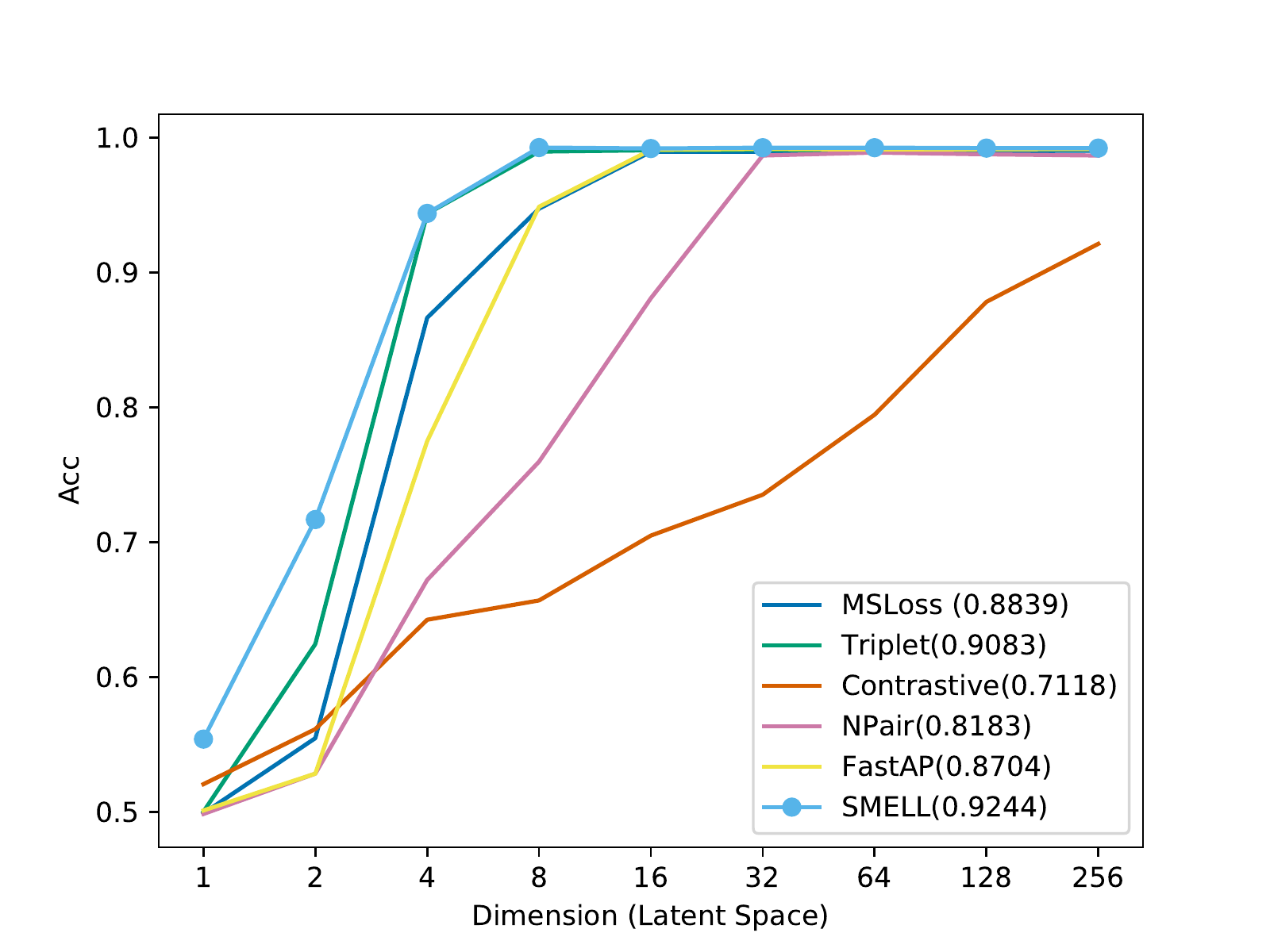}
    \caption{MNIST evaluation for different dimension of latent space. The AUC is reported in parentheses.}
    \label{fig:mnist_acc}
\end{figure}
We noticed that some techniques are highly dependent on the feature extractor. For example, the Contrastive loss~\cite{drlim} was proposed to capture coherent semantics in a latent space. However, the proposal aims to capture the semantics of the data, but, without the aid of convolutions layers, we observe a performance degradation when compared to other techniques.

\subsection{Behavior Analysis}
\label{subsec:behaviour}

    \begin{figure*}[!ht]
        \centering
        \includegraphics[width=0.9\linewidth]{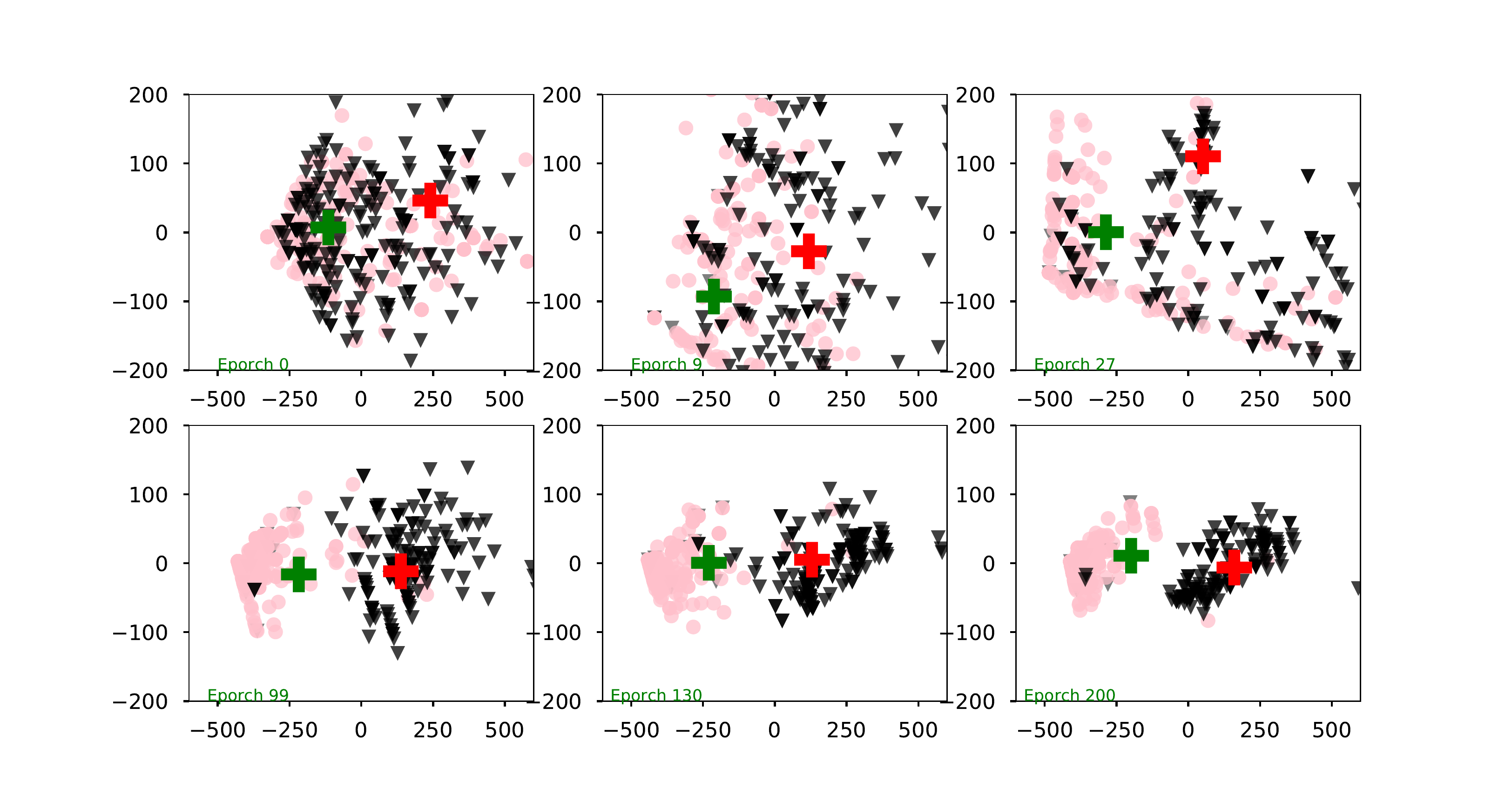}
        \caption{Simultaneous training of  $\mu^+$ (green) and $\mu^-$ (red) markers's position and data representation in S-space for some training epochs.}
        \label{fig:plot_pca_1}
    \end{figure*}

Our proposal is based on optimizing the parameter set  $\bm \Sigma = \{\Theta, \Theta', \bm{M}\}$ using markers 
(with a t-student kernel). 


SMELL learns a representation of input pairs that groups the points with similar and dissimilar labels around their respective markers. We can observe this behavior in Figure~\ref{fig:plot_pca_1}. In this Figure, the input pairs of similar and dissimilar labels are represented by  pink circles and gray triangles, respectively. In addition $\mu^+$  and $\mu^-$ markers are represented by green and red crosses respectively.
We plot some $ s_{ij} $ vectors for input pairs $(x_i, x_y)$ of the test set for the Balance dataset. Initially, after training the autoencoder, a two-dimensional plot was created by the aid of PCA before (first figure) and after (the other figures) the optimization process.

We can observe in fist plot in the Figure~\ref{fig:plot_pca_1} (before adjusting the markers' positions) that the points do not present a well-defined cluster structure. This behavior changes when we analyze the last plot in  Figure~\ref{fig:plot_pca_1} (after adjusting the markers).

In the last plot in Figure~\ref{fig:plot_pca_1}, we can see that there are well-defined groups around the markers. Moreover, by comparing the scale of the Figures~\ref{fig:plot_pca_1}, we see that in the last case, points are more spaced, i.e., our proposal tends to group points around their respective markers. This behavior corroborates our initial hypothesis described in (\textbf{H\ref{hpy:disjoint_regions}}).

We observed that our proposal acts as an attractive potential. In this sense, the marker ``pulls'' the favorable points (similarity mark ``pulls'' similar points). Therefore, their movement resembles a Group Mobility Model~\cite{hong1999group}, i.e., the marker is being positioned, and the points go ``following'' the leader as a ``caravan'' of nomads. At the same time, the markers tend to repulse themselves.

This can be seen as such an intense attracting field, which locks the movement dynamics of the points closest to the markers.

\subsection{Latent space and S-space analysis}
\label{subsec:latent}

For a better understanding of the latent space found by SMELL, we analyzed the behavior of our proposal using the sonar and MNIST datasets as shown in Figure~\ref{fig:mnist_fig} and~\ref{fig:sonar_fig}. 
%
%
 For the sake of visualization, in this analysis, we use the setup discussed in Section~\ref{sec:experiments} with $n=2$ (latent dimension). Figures~\ref{fig:sonar_latented=2} and~\ref{fig:latented=2} show the \textit{latent feature space} (output of encoder). Observe that in these figures, points represent individual objects. Red and blue points represent different classes. There are two classes in sonar dataset, and we show only two classes of MNIST (handcraft digits 4 and 9). These \textit{latent feature spaces} result from the joint optimization process of the autoencoder and the S-space. 
 
 \begin{figure*}
\centering
\begin{subfigure}{.38\textwidth}
  \centering
  \includegraphics[width=1.0\linewidth]{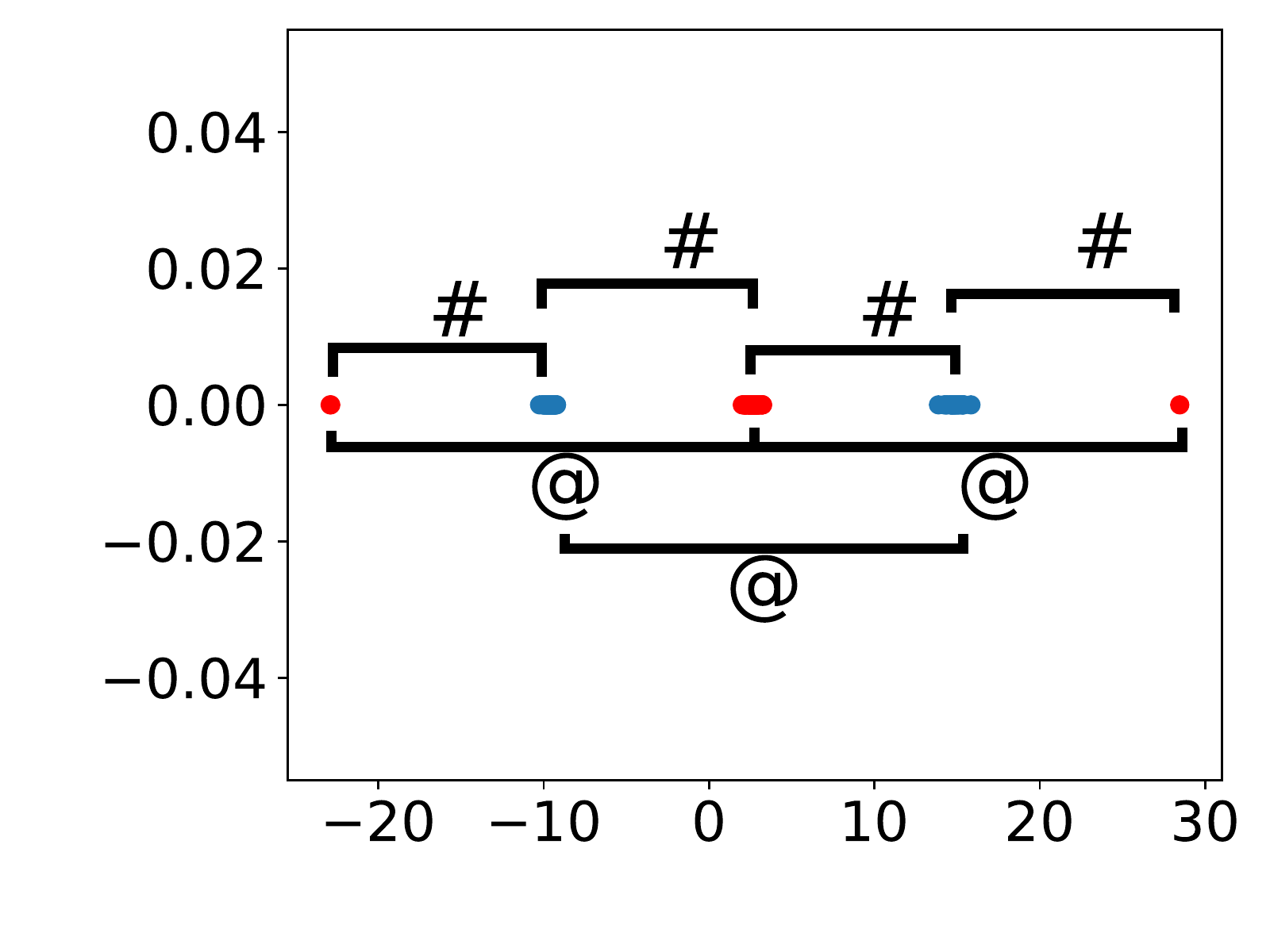}
  \caption{Latent space ($n=2$)}
  \label{fig:sonar_latented=2}
\end{subfigure}
\begin{subfigure}{.38\textwidth}
  \centering
  \includegraphics[width=1.0\linewidth]{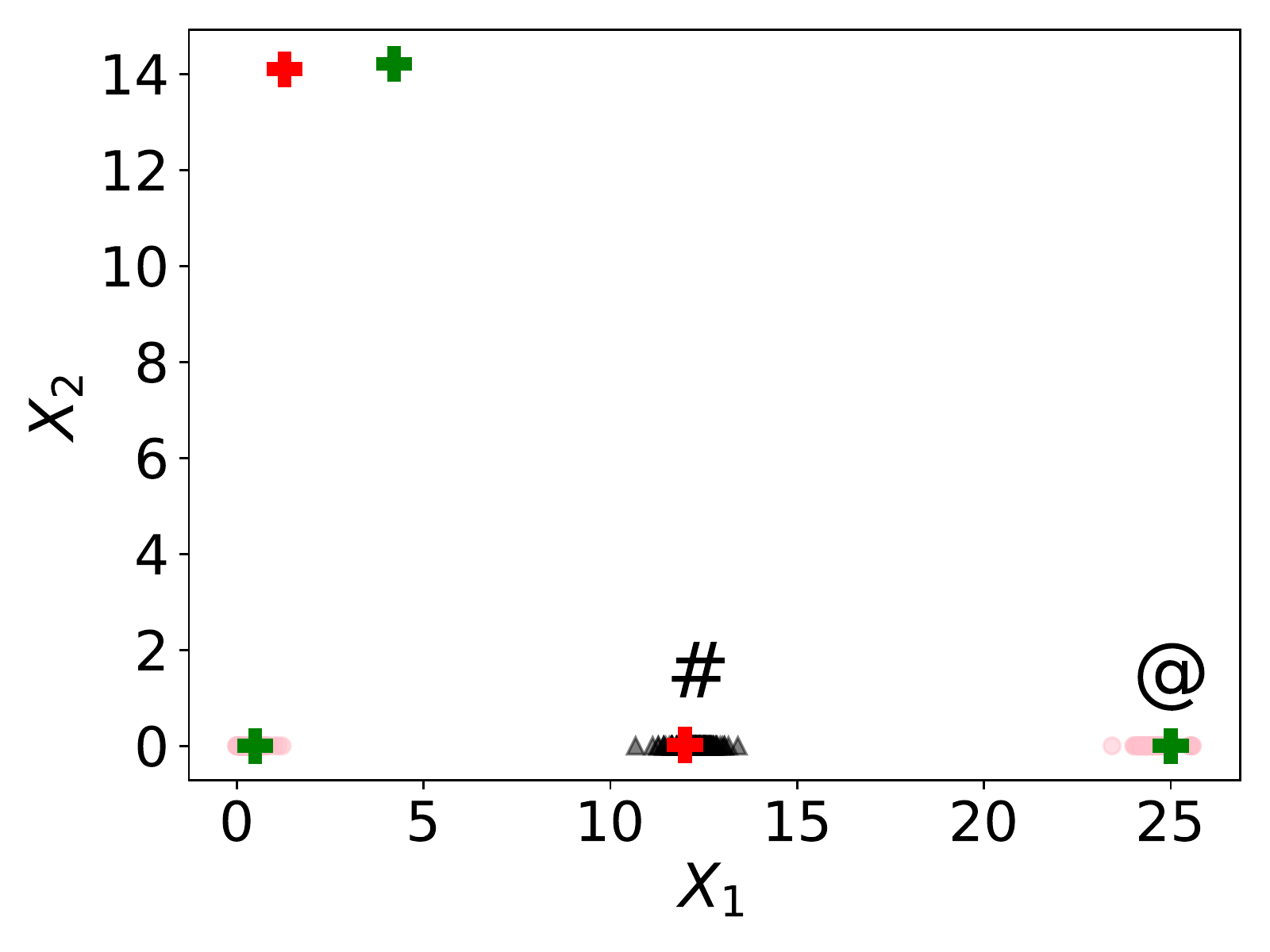}
  \caption{S-space ($n=2$)}
  \label{fig:sonar_fig1}
\end{subfigure}%
\caption{Sonar Datasets: Latent Space and S-space Analysis for SMELL.}
\label{fig:sonar_fig}

\end{figure*}

\begin{figure*}
\centering
\begin{subfigure}{.45\textwidth}
  \centering
  \includegraphics[width=0.88\linewidth]{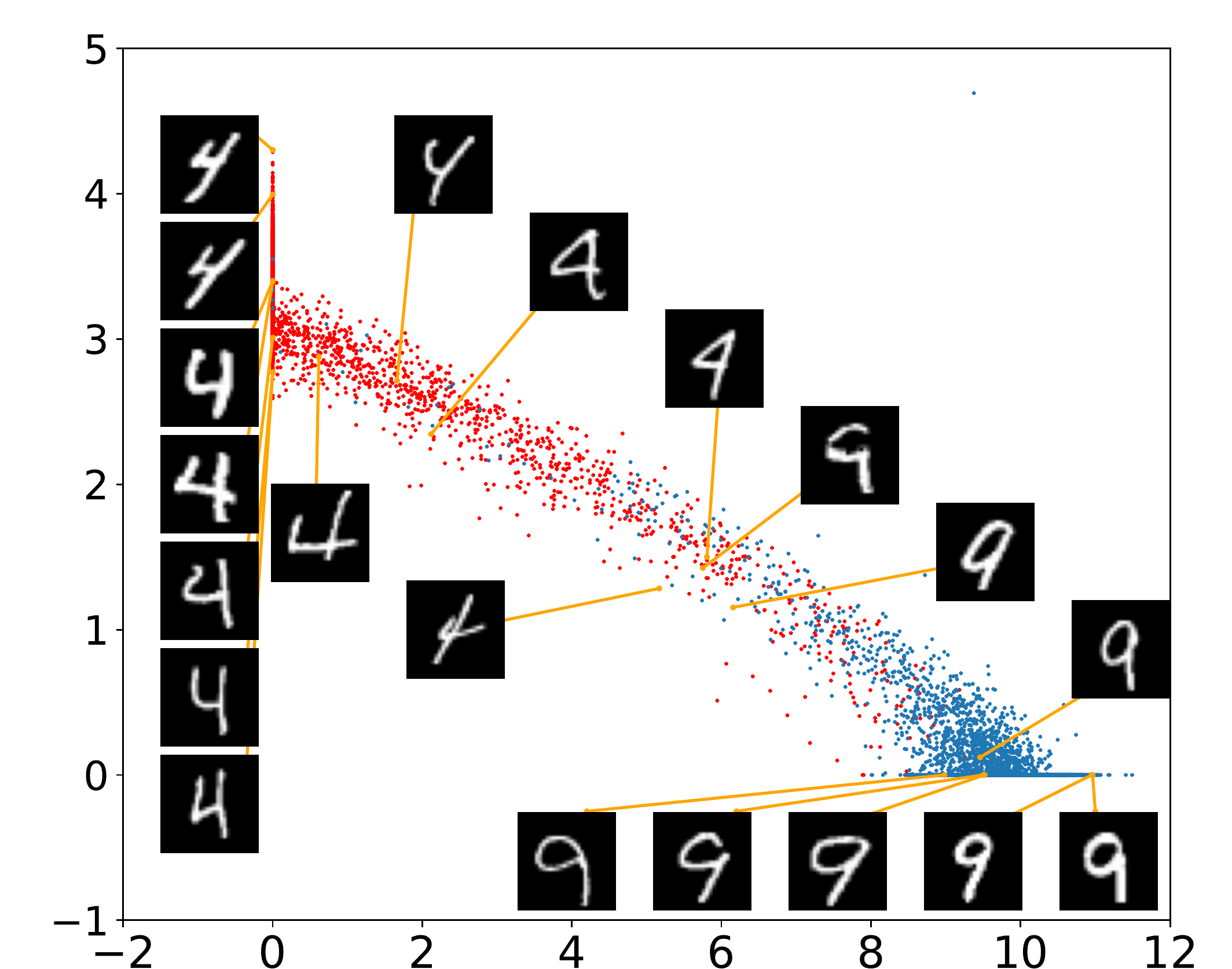}
  \vspace{12pt}
  \caption{Latent space ($n=2$)}
  \label{fig:latented=2}
\end{subfigure}
\begin{subfigure}{.45\textwidth}
  \centering
  \includegraphics[width=1\linewidth]{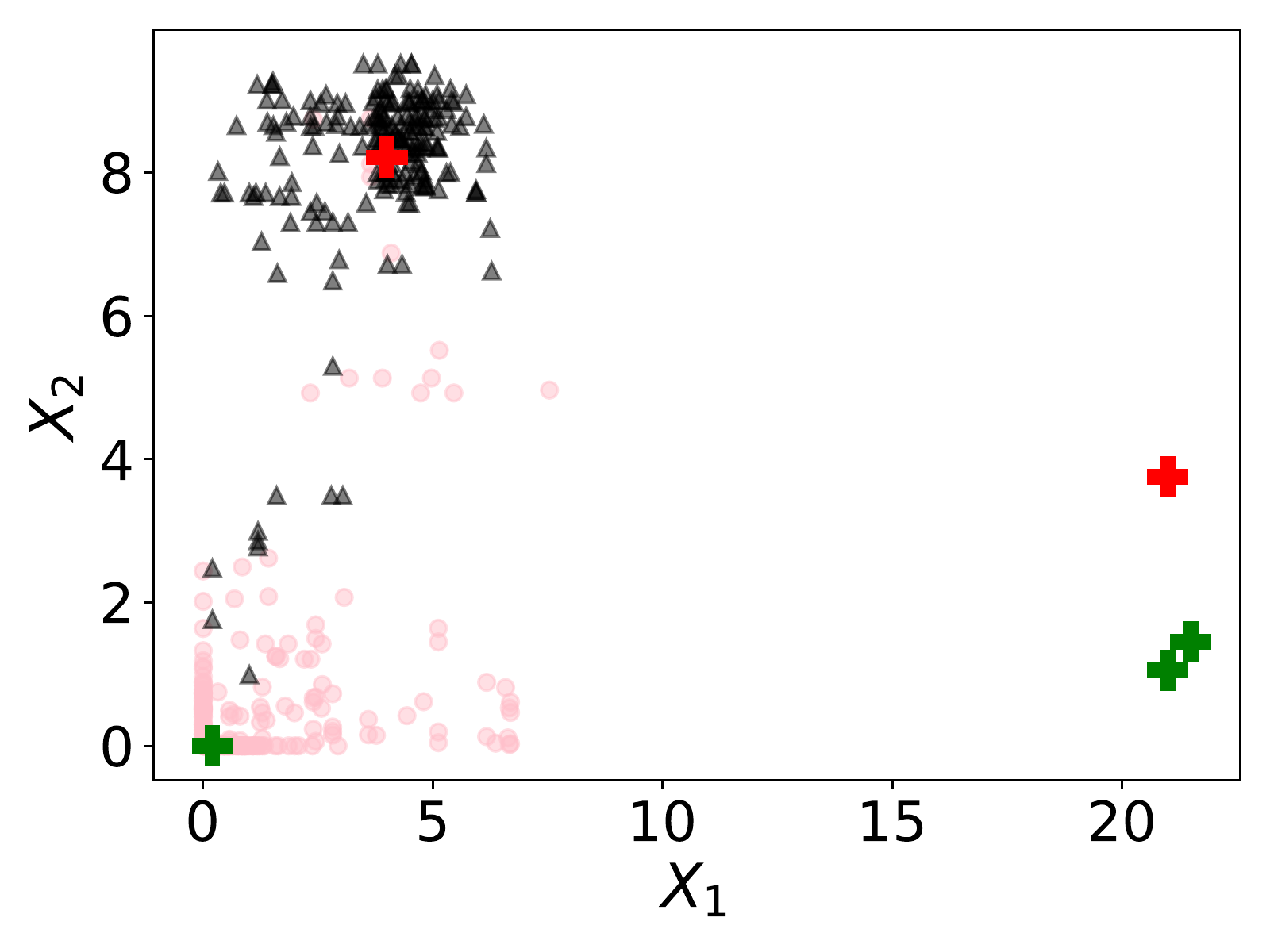}
  \caption{S-space ($n=2$)}
  \label{fig:sfig1}
\end{subfigure}%

\caption{MNIST Datasets: Latent Space and S-space Analysis for SMELL.}
\label{fig:mnist_fig}
\end{figure*}
 
 In Figure~\ref{fig:sonar_latented=2}, we observe that red points are grouped in different regions far apart at a distance approximately constant, denoted by $@$. Similarly, blue points are apart at a distance approximately constant, $@$. Different clusters are apart at a distance approximately constant, denoted as $\#$. 

Figures~\ref{fig:sonar_fig1} and~\ref{fig:sfig1} show a random sample of 400 data pairs from the sonar dataset mapped to S-space (200 similar and 200 dissimilar pairs). In S-space, points represent a pair of objects. Pink circles and black triangles represent similar and dissimilar labels, respectively. Also, similarity and dissimilarity markers are represented by green and red crosses, respectively.

Figure~\ref{fig:sonar_fig1} show some clustered regions. The region grouped by the similarity marker (closer to the origin) is responsible for grouping elements of similar classes with a distance closer to 0. This result corroborates with the Proposition~\ref{pro:mod}. The same behavior is found in Figure~\ref{fig:sfig1}, where we observe a green cross close to the origin.

However, in Figure~\ref{fig:sonar_fig1}, we observe some similar objects mapped to points that have distance close to $@$, instead of zero. The green cross located at $@$ is responsible for creating the similarity region that represents this situation. Other regions of similarity and/or dissimilarity can occur, depending on the data complexity, and are represented by other green/red crosses. Dissimilarity regions are depicted as $\#$. Therefore, in the space found by SMELL, we see the behavior of multiple groups, separated by distances determined by the similarity/dissimilarity markers (labels $\#$ and $@$).

Observe in Figure~\ref{fig:latented=2} the soft transition from digit 4 to 9, which shows that the S-space preserves the connection between these two similar digits. This effect is captured even though we do not use any data-specific feature extractor, such as convolution layers. In SMELL, the encoder can be switched by any feature extractor tailored explicitly for the input data. 

In Figure~\ref{fig:latented=2}, we observe that the handcrafted digits four are grouped (on the left). We observe that even in this group, the similarities between the digits remain. The first two digits in the top-left region correspond to numbers with thicker writing and slightly rotated, and as we go down in the latent space, the shape of the digit starts to become thinner. This behavior indicates a gradient that represents the thickness of the object. This same behavior occurs similarly to digit 9. There is a transition from groups of digits 4 to 9, i.e., there is a semantic in this transition. 
As we move along the diagonal that connects the two groups, gradually, the numbers 4 resemble the number 9, so that, in the middle of the diagonal, it is tough to differentiate between these two numbers. It is also worth noting that, the further away from the denser regions of the points cloud, the less readable are the numbers, for instance, the two digits four depicted below the transition diagonal. We see that our proposal uses markers to help in the convergence and finds a latent space that preserves the semantics of the original data. This behavior corroborates our initial hypothesis described in (\textbf{H\ref{hyp:separability}}).





It is also worth noting that our proposal has no sensitive learning in the presence of multiple markers, i.e., even in this experiment that we have defined three similarity markers and two dissimilarity markers, our proposals does not use all. This behavior is emphasized in Figures~\ref{fig:sonar_fig1} and~\ref{fig:sfig1}
, where our proposal removes excessive markers from the groupings  by locating these markers far away from the data. This behavior is an indication that the number of markers is a virtual parameter of the model. 

We hypothesize that markers group data points considered similar (in our context, which have the same labels) and dissimilar (different labels) in disjoint regions. Figures~\ref{fig:sfig1} and~\ref{fig:sonar_fig1}
show this behavior, where we can see similar and dissimilar groups in distinct (and disjoint) regions in S-space. In addition, we can notice in Figure~\ref{fig:sonar_latented=2} that our proposal allows different clusters for the same class ({\it optimal latent space}), as mentioned in Definition~\ref{ax:ex}.

\section{Conclusion}
\label{sec:conclusion}

In this work, we proposed a Supervised Distance Metric learning Encoder with Similarity Space (SMELL), based on the fact that the distance metrics can be simultaneously learned along with a latent representation of the data and the similarity markers.
We hypothesized that SMELL groups data points consider similar  and increases classes separability. We showed evidences that support our hypothesis by a comprehensive behavior analysis.

We also conducted an extensive validation of our proposal comparing it to many methods over different type of input data. We obtained promising results and, in general context, we got best results.
	
We intend to investigate the possible applications for this type of approach, as well as to use the Proposition~\ref{theo:risk} to build a novel loss function specifically tailored for SMELL.

\appendix
\section{Proofs}
\label{sec:appendix}

\begin{propApp}[\ref{pro:mod}]
In S-space, given $k$ positive markers in the set $ \mathcal{M}^+  $ and $n-k$ negative markers in  $  \mathcal{M}^-$, the latent space found by SMELL, i.e., the estimation of the parameters $\Theta$ of  $f_\Theta$, generates an {\it optimal latent space} if $\exists\; \bm \mu_i \in \mathcal{M}^+$ so that $||\mu_i||_2^2 <  ||\mu_j||_2^2$ for any $\bm \mu_j \in \mathcal{M}^-$.

\end{propApp}

\begin{proof}

Given $\bm x_i $ and $\bm x_j $, SMELL measures the similarity between the entries through the t-student kernel given by $q_{ij}^+ $, so that for $ \bm \Sigma^* $ it follows that $ q_{ij}^+ = 1 \Longleftrightarrow l(x_i) = l(x_j) $.

Since $ f_\Theta $ generates an optimal space, we then have $ \mathbb{E} [|| s_{ij} ||_2] = 0 \Longrightarrow l (x_i) = l (x_j) $, so, it follows that for a optimal latent space, we must have

\begin{equation*}
\begin{split}
q_{ij}^+ &= \frac{\sum_{k \in \mathcal{M}^+}(1 + ||\mu_k||_2^2)^{-1}}{\sum_{k \in \mathcal{M}^+}(1 + ||\mu_k||_2^2)^{-1} + \sum_{s \in \mathcal{M}^-}(1 + ||\mu_s||_2^2)^{-1}} \\ 
& = \frac{1}{1 + \frac{\sum_{s \in \mathcal{M}^-}(1 + ||\mu_s||_2^2)^{-1}}{\sum_{k \in \mathcal{M}^+}(1 + ||\mu_k||_2^2)^{-1}}}
.
\end{split}
\end{equation*}

Hence, if we want $ l(x_i) = l(x_j) $, we should ideally have $ q_{ij}^+$ tends to $1 $. It follows that $ \sum_{s \in \mathcal{M}^-}(1 + ||\mu_s||_2^2)^{-1} < \sum_{k \in \mathcal{M}^+}(1 + ||\mu_k||_2^2)^{-1}$.
Therefore, let  $ \mu^+ $ be the element with the smallest module in the set $ \mathcal{M}^+ $; we then have  $ \sum_{k \in \mathcal{M}^+} (1 + ||\mu_k||_2^2)^{-1} < k (1 + ||\mu^+||_2^2)^{-1} $. Analogously, we can consider $ \mu^- $ as the vector with the largest module in the set $ \mathcal{M}^-$, so, $ \sum_{k \in \mathcal{M}^-} (1 + ||\mu_k||_2^2)^{-1} > (n - k) (1 + ||\mu^-||_2^2)^{-1} $.

We can then conclude that $k (1 + ||\mu^+||_2^2)^{-1} > (n - k) (1 + ||\mu^-||_2^2)^{-1}$, and therefore, $(n - k) (1 + ||\mu^+||_2^2) < k (1 + ||\mu^-||_2^2)$. Furthermore, adding the restriction that SMELL has a similar count of positive and negative markers (section~\ref{subsec:initialization}), we have  $1 + ||\mu^+||_2^2 < 1 + ||\mu^-||_2^2$.

\end{proof}{}

\begin{propApp}[\ref{theo:risk}]
For S-spaces built with one marker in each group, $\mu^+ \in M^+$ and $\mu^-\in M^-$, the misclassification risk function of a positive marker is derived analytically.

\end{propApp}

\begin{proof}

Firstly, we consider  the risk of the similarity in a random negative pair to be more than the similarity in a random positive pair~\cite{hist_loss} as

$$
R = \int_{-\infty}^\infty p^-(x) \left[\int_{-\infty}^{y} p^+(y) dy\right]dx.
$$

%
Consider $ \mathcal{M}^+ = \{\mu^+\} $ and $ \mathcal{M}^- = \{\mu^-\} $, i.e., $\mathcal{M}^+$ and $ \mathcal{M}^-$ have cardinality 1, and $D^+(x_i, x_j)$ and $D^-(x_i, x_j)$ are the euclidean distances from the point $s_{ij}$ to the positive and negative markers, respectively. The risk of misclassification is

\begin{multline*}
R^+ = \int_{0}^{D^+(x_i,x_j)} \bigg[ p^-(D^+(x_i, x_j)) \\ \left( \int_{0}^{y = D^+(x_i, x_j)}   p^+(y) d(y)\right) d(D^+(x_i,x_j)) \Bigg].
\end{multline*}

 Therefore, due to the construction of S-space, we consider that the likelihood of similarity/dissimilarity between the $s_{ij}$ representation of two samples is calculated as the relative distance to a marker. Due to this construction, we have
 
$$
R^+ =  \int_{0}^{D^+(x_i,x_j)} q^- \left[ \int_{0}^{y = D^+(x_i, x_j)} q^+ d(y)\right] d(D^+(x_i,x_j)) .
$$

Calculating each term separately, we have that for the markers in the sets $ \mathcal{M}^+ $ and $ \mathcal{M}^- $, the probability of a having a similar objects is

\begin{equation*}
\begin{split}
q^+ &= \frac{(1+D^+(x_i,x_j)^2)^{-1}}{(1+D^+(x_i,x_j)^2)^{-1} + (1+D^-(x_i,x_j)^2)^{-1}} \\
&= \frac{1+D^-(x_i, x_j)^2}{2+D^+(x_i, x_j)^2+D^-(x_i, x_j)^2}.
\end{split}
\end{equation*}

Analogously, we can find that $ q ^ - = 1 - q ^ + $. 
To simplify the notation, consider that $ D^+ $ equals $ D^+(x_i, x_j) $ and $ D^- $ equals $ D^-(x_i, x_j) $.
We have for $ R^+$

\begin{equation*}
\begin{split}
R^+ &=  \int_{0}^{D^+} q^- \left[ \int_{0}^{y = D^+} q^+ d(y)\right]d(D^+) \\ 
&= \int_{0}^{D^+} (1-q^+) \left[\int_{0}^{y = D^+} q^+ d(y)\right] d(D^+).
\end{split}
\end{equation*}

Therefore, the risk of misclassification for the positive marker can be reduced to $ R^+ = \int_0^{D^+} (1-q ^ +) \Phi (D^+) dD^+ $, where $ \Phi $ is the cumulative density function (CDF) for $ q^+ $. 
%

Calculating the accumulated histogram $$\Phi(x) = \int_0^{x} \frac{1+(D^-)^2}{2+(x)^2+(D^-)^2 }dx,$$ we get
$$
\Phi^+(x) = \frac{((D^-)^2+1)tan^{-1}\left(\frac{x}{\sqrt{(D^-)^2+2}}\right)}{\sqrt{(D^-)^2+2}} .
$$

Therefore, by solving the integral $ \int_0^{D^+} (1-q^+) \Phi (D^+) dD^+ $, we have the exact analytical value of the misclassification risk function for the positive marker. With that, it follows that


\begin{multline*}
R^+ = \frac{(D^-)^2+1}{\sqrt{(D^-)^2+2}}\Bigg[\left(\sqrt{(D^-)^2+2}\right)log\left(\frac{1}{\sqrt{\frac{(D^+)^2}{(D^-)^2+2}+1}}\right) -\\- \frac{((D^-)^2+1)\left[tan^{-1}\left( \frac{(D^+)}{\sqrt{(D^-)^2+2}} \right)\right]^2}{\sqrt{(D^-)^2+2}} +\\+ (D^+)tan^{-1}\left(\frac{(D^+)^2}{\sqrt{(D^-)^2+2}}\right)\Bigg].
\end{multline*}

\end{proof}

\section{Ablation Results}
\label{sec:appendix_ablation}

This appendix show full comparison of ablation study for 27 different datasets.

\begin{table}[htpb]
\centering
\tiny
\begin{tabular}{p{1.74cm}
>{\columncolor{mygray}}c c
>{\columncolor{mygray}}c c
>{\columncolor{mygray}}c c
>{\columncolor{mygray}}c c
>{\columncolor{mygray}}c c}
\hline
    \multirow{2}{*}{Dataset} & SMELL  & SMELL  & SMELL & SMELL  &  SMELL 
     \\ 
     & $(r_r=0)$ & ($r_d=0$)  & ($r_r=r_d=0$ )   & (Euclidian) & (S-space)   \\
     \hline
     
Appendicitis & 78.90 $\pm$ 11.12 &   79.09 $\pm$ 11.02 & 79.09 $\pm$ 09.59 & 77.18 $\pm$ 9.13 & \textbf{80.19 $\pm$ 7.74} 
\\
Balance & 97.00 $\pm$ 01.03     &   97.00 $\pm$ 01.12 & 97.00 $\pm$ 1.02 & 98.40 $\pm$ 1.22 &  \textbf{98.88 $\pm$ 1.02} 
\\
Banana (10\%) & 89.44 $\pm$ 4.89 &   90.76 $\pm$ 08.83 & 90.01 $\pm$ 3.76  & 87.73 $\pm$ 1.14 &  \textbf{90.95 $\pm$ 4.42} 
\\
Bupa & 64.58 $\pm$ 09.92 &   \textbf{67.05 $\pm$ 09.62} & 66.37 $\pm$ 8.83 & 55.15 $\pm$ 10.12 &  63.92 $\pm$ 6.85 
\\
Cleveland & \textbf{52.22 $\pm$ 07.42} &   51.20 $\pm$ 05.65  & 52.21 $\pm$ 6.32 & 51.93 $\pm$ 7.22 & 51.24 $\pm$ 8.49 
\\
Glass & 67.52 $\pm$ 14.71 &   67.81 $\pm$ 12.16 & 66.91 $\pm$ 11.31 & \textbf{70.75 $\pm$ 11.46} &  66.94 $\pm$ 13.24 
\\
Ionosphere & \textbf{89.75 $\pm$ 5.14} &  88.89 $\pm$ 02.37 & \textbf{89.75 $\pm$ 4.98} & 84.88 $\pm$ 6.79 &   89.47  $\pm$ 4.59 
\\
Iris & 95.33 $\pm$ 04.00 &   94.67 $\pm$ 04.47 & 94.67 $\pm$ 4.27 & 95.33 $\pm$ 4.27 & \textbf{96.00 $\pm$ 3.26} 
\\
Letter (10\%)& \textbf{80.30 $\pm$ 3.72} &  77.58 $\pm$ 08.40 & 77.77 $\pm$ 3.82 & 62.96 $\pm$ 8.62 &  77.77 $\pm$ 4.72 
\\
Magic (10\%) & 84.44 $\pm$ 02.98&  84.44 $\pm$ 02.74 & \textbf{84.49 $\pm$ 2.37} & 78.92 $\pm$ 7.11 &   83.49 $\pm$ 2.58 
\\
Monk-2 & \textbf{100 $\pm$ 0.00} &  \textbf{100.00 $\pm$ 0.00} & \textbf{100 $\pm$ 0.00} & \textbf{100 $\pm$ 0.00} & \textbf{100 $\pm$ 0.00} 
\\
Movement-libras & 86.21 $\pm$ 04.44 &  84.72 $\pm$ 04.44 & 85.56 $\pm$ 5.39 & 83.33 $\pm$ 6.92 & \textbf{87.78 $\pm$ 3.61}
\\
Newthyroid & \textbf{97.71 $\pm$ 02.25} & \textbf{97.71 $\pm$ 03.57} & 96.77 $\pm$ 2.29 & 90.71 $\pm$ 10.97 &  96.77 $\pm$ 2.11 
\\
Phoneme (10\%) & 82.77 $\pm$ 03.89  & \textbf{83.87 $\pm$ 04.97} & 82.21 $\pm$ 3.04 & 81.29 $\pm$ 4.89 & 81.48 $\pm$ 4.31 
\\
Pima & 70.44 $\pm$ 05.17 & 70.98 $\pm$ 05.96 & \textbf{71.75 $\pm$ 2.65} & 69.68 $\pm$ 3.87 & 70.44 $\pm$ 3.52
\\
Ring (10\%) & \textbf{89.32 $\pm$ 21.56} & 65.36 $\pm$ 19.21 & 66.10 $\pm$ 3.11 & 71.80 $\pm$ 16.55 &  89.22 $\pm$ 4.56 
\\
Satimage (10\%) & 83.49 $\pm$ 02.94 & 85.20 $\pm$ 03.12 & 84.58 $\pm$ 4.45 & \textbf{85.37 $\pm$ 3.11} & 84.11 $\pm$ 3.57 
\\   
Segment (10\%) & \textbf{90.48 $\pm$ 04.48} &  89.05 $\pm$ 07.22 & 89.05 $\pm$ 4.42 & 90.27 $\pm$ 4.76 &  89.76 $\pm$ 3.96 
\\
Sonar & 84.05 $\pm$ 10.53 &  \textbf{85.52 $\pm$ 10.92} & 84.55 $\pm$ 10.44 & 81.76 $\pm$ 12.88 & 83.59 $\pm$ 11.26 
\\
Thyroid (10\%) & 95.59 $\pm$ 02.43  &  95.56 $\pm$ 02.15 & \textbf{95.69 $\pm$ 2.51} & 94.71 $\pm$ 2.24 &  94.71 $\pm$ 2.45 
\\
Titanic (10\%) & 62.41 $\pm$ 16.07 &  62.56 $\pm$ 05.23 & 63.76 $\pm$ 16.22 & \textbf{73.13 $\pm$ 6.13} & 66.97 $\pm$ 15.85 
\\
Twonorm (10\%) & 96.00 $\pm$ 01.41 &  96.75 $\pm$ 06.07 & 97.00 $\pm$ 1.48 & 93.78 $\pm$ 13.33 & \textbf{97.43 $\pm$ 1.76 }
\\
Vehicle & 84.40 $\pm$ 02.00 &  83.34 $\pm$ 12.02 & \textbf{84.87 $\pm$ 3.33} & 78.32 $\pm$ 18.33 &  84.40 $\pm$ 2.63 
\\
Vowel & 98.12 $\pm$ 00.79 & 98.12 $\pm$ 15.63 & \textbf{98.99 $\pm$ 1.09} & 98.48 $\pm$ 1.04 &  \textbf{98.99 $\pm$ 0.90}
\\
Wdbc & \textbf{96.65 $\pm$ 02.99} & 96.47 $\pm$ 02.40 & \textbf{96.65 $\pm$ 2.99} & 89.11 $\pm$ 14.63 & \textbf{96.65 $\pm$ 2.89 }
\\
Wine & 97.77 $\pm$ 03.59 & 97.71 $\pm$ 05.12 & \textbf{98.30 $\pm$ 5.11} & 97.19 $\pm$ 5.14 & 97.77 $\pm$ 5.11
\\
Wisconsin & \textbf{96.21 $\pm$ 02.14} & 95.91 $\pm$ 01.84 & 95.62 $\pm$ 1.87 &  95.91 $\pm$ 2.48 &  \textbf{96.21 $\pm$ 2.26} 
\\
\hline
Accuracy\_AVG & 0.8254 & 0.8169 & 0.8178 & 0.7994 & \textbf{0.8268}
\\
Ranking\_AVG & \textbf{2.2857} & 2.8571 & \textbf{2.2857} & 3.6429 & \textbf{2.2857 }
\\
Diff\_AVG & 0.0116 & 0.0201 & 0.0193 & 0.0376 & \textbf{0.0102}
\\
\# of 1$^{\text{st}}$ & 9 & 5 & 9 & 4 & \textbf{10}
\\
\hline
\end{tabular}
\caption{Performance comparison of ablation study when using  KNN classification for 27 different datasets.}
\label{tab:ablation_full}
\end{table}
 \pagebreak

\section{Performance Comparison}
\label{app_results}

This appendix show full comparison of all distance metric learning using in this work.

\begin{sidewaystable*}[htbp]
\centering
\tiny
\begin{tabular}{p{1.7cm}
>{\columncolor{mygray}}c c
>{\columncolor{mygray}}c c
>{\columncolor{mygray}}c c
>{\columncolor{mygray}}c c
>{\columncolor{mygray}}c c
>{\columncolor{mygray}}c c
>{\columncolor{mygray}}c l}
\hline
    Dataset &
    ANMM\cite{anmm} &
    KDMLMJ\cite{NGUYEN2017215}   &
    Contrastive\cite{drlim} &
    MSLoss\cite{wang2019multi}  & Triplet\cite{Schroff_2015_CVPR} & NCA\cite{NIPS2004_2566}   & NPair\cite{Nparloss}    & FastAP\cite{Fastaploss}    &
    Euclidian &
    SMELL 
     \\ \hline
     
Appendicitis & 84.27 $\pm$ 10.64 &\textbf{85.09 $\pm$ 9.94} & 84.18 $\pm$ 9.95 & 84.09 $\pm$ 11.02 & 81.27 $\pm$ 10.79 & 84.09 $\pm$ 9.39 & 84.90 $\pm$ 08.61 & 84.09 $\pm$ 10.07 & 84.27 $\pm$ 10.64 & 81.09 $\pm$ 7.74 
\\
Balance & 80.81 $\pm$ 4.22 &79.83 $\pm$ 04.32 & 78.93 $\pm$ 23.54    &   98.24 $\pm$ 01.12 & 96.17 $\pm$ 2.03 & 95.84 $\pm$ 2.28 & 89.74 $\pm$ 08.00&  98.24 $\pm$ 01.31 & 80.16 $\pm$ 5.06  & \textbf{98.88 $\pm$ 1.02}    
\\
Banana (10\%) & 56.18 $\pm$ 12.43 &55.62 $\pm$ 13.25 & 89.06 $\pm$ 12.97 & 84.38 $\pm$ 08.83 & 90.19 $\pm$ 3.85 & 86.79 $\pm$ 2.15 & 89.44 $\pm$ 05.25 & 72.59 $\pm$ 13.55 & 56.18 $\pm$ 12.43  & \textbf{90.95 $\pm$ 4.42} 
\\
Bupa & 60.95 $\pm$ 9.04 & 65.00 $\pm$ 7.888 & 49.63 $\pm$ 5.01 &  67.05 $\pm$ 09.62 & \textbf{68.63 $\pm$ 5.28} & 57.79 $\pm$ 6.62 & 57.16 $\pm$ 08.84 & 58.77 $\pm$ 11.71 & 64.95 $\pm$ 9.04 & 63.92 $\pm$ 6.85 
\\
Cleveland & 55.14 $\pm$ 6.92 & 48.83 $\pm$ 06.64 & 55.54 $\pm$ 2.76 & 54.55 $\pm$ 05.65  & 56.92 $\pm$ 6.92 & 50.23 $\pm$ 6.76 & \textbf{56.99 $\pm$ 06.61} & 55.80 $\pm$ 8.29 & 55.13 $\pm$ 6.92 & 51.24 $\pm$ 8.49 
\\
Glass & 68.13 $\pm$ 10.26 & \textbf{68.95 $\pm$ 10.78} & 51.11 $\pm$ 17.89 & 66.71 $\pm$ 12.16 & 68.30 $\pm$ 11.82 & 67.02 $\pm$ 10.40 & 60.66 $\pm$ 08.20 & 64.65 $\pm$ 9.33 & 68.13 $\pm$ 10.25 & 66.94 $\pm$ 13.24 
\\
Ionosphere & 85.18 $\pm$ 4.92 & 84.04 $\pm$ 03.43 & 86.16 $\pm$ 19.30 & \textbf{94.02 $\pm$ 02.37} & 92.89 $\pm$ 2.89 & 88.88 $\pm$ 4.15 & 86.33 $\pm$ 07.32 & 92.60 $\pm$ 3.39 &85.18 $\pm$ 4.92 & 89.47  $\pm$ 4.59 
\\
Iris & 94.00 $\pm$ 4.67 & 96.00 $\pm$ 04.00 & \textbf{96.67 $\pm$ 4.47} & \textbf{96.67 $\pm$ 04.47} & 96.00 $\pm$ 2.72 & 95.33 $\pm$ 4.27 & 94.67 $\pm$ 04.00 & \textbf{96.67 $\pm$ 3.33}  & 94.00 $\pm$ 4.67 & 96.00 $\pm$ 3.26 
\\
Letter (10\%) & 78.93 $\pm$ 10.95 & \textbf{85.7 $\pm$ 11.83} & 25.95 $\pm$ 21.15 & 77.64 $\pm$ 08.40 & 46.54 $\pm$ 23.79 & 61.55 $\pm$ 21.08 & 46.10 $\pm$ 08.06 & 46.55 $\pm$ 23.79 & 79.13 $\pm$ 8.95 & 77.77 $\pm$ 4.72 
\\
Magic (10\%) & 62.28 $\pm$ 9.68 & 77.34 $\pm$ 09.54 & 73.5 $\pm$ 5.57 & \textbf{84.65 $\pm$ 02.74} & 84.07 $\pm$ 2.64 & 67.99 $\pm$ 9.85 & 82.07 $\pm$ 02.58 & 84.57 $\pm$ 2.60 & 62.82 $\pm$ 9.68 & 83.49 $\pm$ 2.58 
\\
Monk-2 & 95.89 $\pm$ 3.92 &\textbf{100 $\pm$ 0.00} & 71.44 $\pm$ 12.13 & 96.52 $\pm$ 03.29 & 98.37 $\pm$ 1.48 & \textbf{100 $\pm$ 0.00} & 95.21 $\pm$ 04.92 & 98.86 $\pm$ 1.52 & 98.18 $\pm$ 2.72 & \textbf{100 $\pm$ 0.00} 
\\
Movement-libras & 80.83 $\pm$ 3.39 & 87.22 $\pm$ 05.28 & 53.33 $\pm$ 21.64 &  82.78 $\pm$ 04.44 & 75.00 $\pm$ 10.54 & 81.67 $\pm$ 4.51 & 33.61 $\pm$ 11.81 & 67.78 $\pm$ 20.38 & 80.12 $\pm$ 3.39 & \textbf{87.78 $\pm$ 3.61}  
\\
Newthyroid & 95.36 $\pm$ 2.95 & 94.87 $\pm$ 03.58 & 97.25 $\pm$ 3.16 & 96.77 $\pm$ 03.57 & 95.84 $\pm$ 3.84 & \textbf{97.73 $\pm$ 3.05} & 96.32 $\pm$ 04.95 & 97.25 $\pm$ 3.65 & 95.37 $\pm$ 2.95  & 96.77 $\pm$ 2.11 
\\
Phoneme (10\%) & 81.91 $\pm$ 8.87 &81.91 $\pm$ 08.70 & 70.16 $\pm$ 9.79 & 82.39 $\pm$ 04.97 & 79.96 $\pm$ 5.88 & \textbf{83.4 $\pm$ 9.81} & 75.73 $\pm$ 04.41 & 76.64 $\pm$ 8.12 & 62.71 $\pm$ 8.87 & 81.48 $\pm$ 4.31  
\\
Pima & 72.93 $\pm$ 4.30 & 69.28 $\pm$ 04.04 & 64.46 $\pm$ 4.65 & 74.01 $\pm$ 05.96 & \textbf{74.11 $\pm$ 3.96} & 70.44 $\pm$ 3.78 & 72.80 $\pm$ 05.22 & 72.55 $\pm$ 3.94 & 72.93 $\pm$ 4.30 & 70.44 $\pm$ 3.52 
\\
Ring (10\%) & 51.70 $\pm$ 6.39 & 59.61 $\pm$ 07.46 & 73.15 $\pm$ 11.33 & 79.45 $\pm$ 19.21 & \textbf{93.91 $\pm$ 4.84} & 51.19 $\pm$ 9.57 & 82.71 $\pm$ 04.11 & 70.10 $\pm$ 20.91 & 51.77 $\pm$ 6.39 & 89.22 $\pm$ 4.56 
\\
Satimage (10\%) & 84.18 $\pm$ 9.65 & 74.58 $\pm$ 09.15 & 59.50 $\pm$ 17.88 & \textbf{86.35 $\pm$ 03.12} & 82.89 $\pm$ 4.25 & 64.93 $\pm$ 21.39 & 69.82 $\pm$ 11.49 & 83.99 $\pm$ 4.72 & 52.93 $\pm$ 19.66 & 84.11 $\pm$ 3.57 
\\   
Segment (10\%) & 73.12 $\pm$ 12.48 &83.69 $\pm$ 04.05 & 67.11 $\pm$ 17.72 & 84.76 $\pm$ 07.22 & 92.21 $\pm$ 4.68 & 64.88 $\pm$ 15.37 & 57.02 $\pm$ 16.15 & \textbf{92.74 $\pm$ 3.85} & 53.13 $\pm$ 22.48 & 89.76 $\pm$ 3.96 
\\
Sonar & 83.07 $\pm$ 10.58 &82.09 $\pm$ 12.12 & 59.26 $\pm$ 10.03 & 86.00 $\pm$ 10.92 & \textbf{88.43 $\pm$ 9.13} & 86.50 $\pm$ 9.55 & 79.33 $\pm$ 11.93 & 85.05 $\pm$ 10.91 & 82.70 $\pm$ 10.57 & 83.59 $\pm$ 11.26 
\\
Thyroid (10\%) & 90.99 $\pm$ 2.07 &92.65 $\pm$ 02.21 & 91.68 $\pm$ 1.79 & 94.48 $\pm$ 02.15 & 93.69 $\pm$ 2.19 & 90.14 $\pm$ 2.18 & 93.74 $\pm$ 02.50 & 64.17 $\pm$ 31.81 & 90.99 $\pm$ 2.07 & \textbf{94.71 $\pm$ 2.45} 
\\
Titanic (10\%) & 61.74 $\pm$ 16.72 &71.82 $\pm$ 12.67 & 73.13 $\pm$ 5.05 & 72.67 $\pm$ 05.23 & 73.18 $\pm$ 5.16 & 64.16 $\pm$ 9.05 & 72.67 $\pm$ 05.34 & 73.61 $\pm$ 5.78 & \textbf{74.14 $\pm$ 10.96} &  66.97 $\pm$ 15.85 
\\
Twonorm (10\%) & 93.83 $\pm$ 4.54 & 95.95 $\pm$ 04.39 & 91.89 $\pm$ 10.69 & 95.95 $\pm$ 06.07 & 96.62 $\pm$ 1.38 & 95.95 $\pm$ 8.16 & \textbf{98.11 $\pm$ 01.38} & 93.78 $\pm$ 10.22 & 97.5 $\pm$ 4.59 & 97.43 $\pm$ 1.76 
\\
Vehicle & 70.21 $\pm$ 3.66 & 65.95 $\pm$ 03.46 & 40.23 $\pm$ 10.55  & 74.13 $\pm$ 12.02 & 81.91 $\pm$ 4.02 & 74.71 $\pm$ 3.13 & 45.17 $\pm$ 08.40 & \textbf{85.58 $\pm$ 3.45} &  65.95 $\pm$ 3.66 & 84.40 $\pm$ 2.63 
\\
Vowel & 97.77 $\pm$ 0.98 & 98.28 $\pm$ 01.37&  88.69 $\pm$ 5.58 &57.37 $\pm$ 15.63& 95.25 $\pm$ 3.10 & 97.68 $\pm$ 1.20 & 72.12 $\pm$ 06.70 & 78.79 $\pm$ 10.75 & 98.28 $\pm$ 0.98 &  \textbf{98.99 $\pm$ 0.90}
\\
Wdbc & 96.48 $\pm$ 2.49 & 92.79 $\pm$ 03.32 & 95.61 $\pm$ 4.72 &\textbf{97.36 $\pm$ 02.40} & 95.18 $\pm$ 1.97 & 92.44 $\pm$ 3.16 & 93.86 $\pm$ 07.49 & 96.48 $\pm $2.72 & 92.79 $\pm$ 2.49 & 96.65 $\pm$ 2.89 
\\
Wine & 95.52 $\pm$ 4.17 & 97.71 $\pm$ 02.80 & 95.51 $\pm$ 4.72 & 97.22 $\pm$ 05.12 & 82.56 $\pm$ 2.22 & 97.44 $\pm$ 2.54 & 89.25 $\pm$ 13.83 & 96.63 $\pm$ 5.11 & 69.18 $\pm$ 4.16  &
\textbf{97.77 $\pm$ 5.11} 
\\
Wisconsin & 96.52 $\pm$ 2.79 &\textbf{96.66 $\pm$ 02.71} & 74.78 $\pm$ 2.51 & 96.05 $\pm$ 01.84 & 96.05 $\pm$ 2.36 &  95.82 $\pm$ 1.73& 95.92 $\pm$ 02.58 & 96.22 $\pm$ 2.33 & 96.39 $\pm$ 2.94 & 96.21 $\pm$ 2.26 
\\
\hline
Accuracy\_AVG & 0.7768 & 0.7827 & 0.6923 & 0.8076 & 0.8115 & 0.7732 & 0.7380 & 0.7801 & 0.7486 & \textbf{0.8268 }
\\
Ranking\_AVG & 5.71429 & 4.96429 & 7.17857 & 3.7857 & 3.9643 & 5.67857 & 6.2143 & 4.7857 & 5.89286 & \textbf{4.3214} 
\\
Diff\_AVG & 0.0726 & 0.0667 & 0.1571 & 0.0418 & 0.0379 & 0.0762 & 0.1114 & 0.0693 & 0.1007 & \textbf{0.0242} 
\\
\# of 1$^{\text{st}}$& 0 & 5 & 1 & 5 & 4& 3 & 2 & 3 & 1 & \textbf{7} 
\\
\hline
\end{tabular}
\caption{Performance comparison of some distance metrics approaches and SMELL when using  KNN classification for 27 different datasets.}
\label{tab:results_p1_full}
\end{sidewaystable*}
\pagebreak

\bibliography{elsarticle-template}

\end{document}